\documentclass[accepted]{uai2025} 
                        
\usepackage[american]{babel}

\usepackage{natbib} 
\usepackage{mathtools} 
\usepackage{booktabs} 
\usepackage{multirow}
\usepackage{tikz} 

\usepackage{algorithm}
\usepackage[noend]{algorithmic}
\usepackage{amsfonts}
\usepackage{bbold}
\usepackage{graphicx}
\usepackage{amsthm,amsmath,bm}
\usepackage{cleveref}

\newtheorem{theorem}{Theorem}
\newtheorem{definition}{Definition}

\usepackage{soul}

\newcommand{\ZbYx}{{\bm{Z}\hspace{-.5pt}Y\mspace{-1mu}(x) \hspace{-.5pt}}}

\newcommand{\YxIZb}{{\hspace{-.5pt}Y\mspace{-1mu}(x) \hspace{-.5pt} | \hspace{-.5pt} \bm{Z}}}

\newcommand{\YIZbX}{{Y \hspace{-.5pt} | \hspace{-.5pt}\bm{Z} \hspace{-.5pt}X}}

\newcommand{\ZbXY}{{\hspace{-.5pt}\bm{Z} \hspace{-1pt} X \hspace{-.75pt} Y}}

\newcommand{\ZbX}{{\hspace{-.5pt} \bm{Z} \hspace{-1pt} X}}

\newcommand{\Yx}{{Y\mspace{-1mu}(x)}}

\newcommand{\cmid}{\,|\,}

\title{Testing Generalizability in Causal Inference}

\author[1]{\href{mailto:<manela@stats.ox.ac.uk>?Subject=Your UAI 2025 paper}{Daniel~de~Vassimon~Manela\thanks{Equal contribution. }}{}}
\author[1]{\href{mailto:<linying.yang@stats.ox.ac.uk>?Subject=Your UAI 2025 paper}{Linying~Yang$^{\ast}$}}
\author[1,2]{Robin~J.~Evans}
\affil[1]{%
    Department of Statistics, University of Oxford
}
\affil[2]{%
    Pioneer Centre for SMARTbiomed, University of Oxford
}
  
\begin{document}

\maketitle
\begin{abstract}
Ensuring robust model performance in diverse real-world scenarios requires addressing generalizability across domains with covariate shifts. However, no formal procedure exists for statistically evaluating generalizability in machine learning algorithms. 
Existing predictive metrics like mean squared error (MSE) help to quantify the relative performance between models, but do not directly answer whether a model can or cannot generalize. 
To address this gap in the domain of causal inference, we propose a systematic framework for statistically evaluating the generalizability of high-dimensional causal inference models. Our approach uses the frugal parameterization to flexibly simulate from fully and semi-synthetic causal benchmarks, offering a comprehensive evaluation for both mean and distributional regression methods. 
Grounded in real-world data, our method ensures more realistic evaluations, which is often missing in current work relying on simplified datasets. 
Furthermore, using simulations and statistical testing, our framework is robust and avoids over-reliance on conventional metrics, providing statistical safeguards for decision making. 
\end{abstract}

\section{Introduction}\label{sec:intro}
Model generalizability has garnered significant interest in causal inference \citep{bareinboim2016causal,curth2021really, johansson2018learning, buchanan2018generalizing,ling2022critical,bica2022transfer}. This encompasses transportability under covariate shifts between domains and extrapolation. In causal inference, it specifically refers to the ability of a causal model to make accurate predictions or draw valid conclusions when applied to a domain different from the one it was trained on. This concept is crucial when the objective involves understanding and predicting the effects of interventions across various settings. It holds particular importance in clinical contexts, where the 
interest in personalized treatment and patient stratification underscores the need to generalize inferences across diverse populations.

Current approaches for evaluating model generalizability generally involve using predictive metrics like AUC for classification or mean squared error for regression \citep{zhou2022domain,yu2024survey}. 
However, these metrics do not directly assess the evidence for whether a model generalizes across domains, nor do they provide error-controlled decision thresholds.
Does an MSE of 5 on another domain imply that the model does not generalize? How about an MSE of 1? Are these results and interpretations reproducible with statistical guarantees? How much does random noise affect these metrics? These are critical problems that should be carefully considered in causal inference questions involving multiple domains. It is essential to establish a systematic evaluation framework for generalizability performance, which offers a robust, reproducible evaluation of model performance on relevant tasks.

One approach to this problem is statistical testing, where we set the question of interest as the hypothesis we test against. However, it is difficult to obtain power against a wide-range of alternative hypotheses when performing tests conditional on a high-dimensional covariate set. This is a problem for causal practitioners as they are often interested in modeling quantities such as the individual treatment effect. 

    
\paragraph{Main Contributions} We propose a systematic framework for statistically  evaluate the generalizability of high-dimensional causal inference algorithms by targeting low-dimensional causal margins. 
Complementing existing predictive metrics such as MSE, we provide a testing framework that statistically evaluates the transportability of both mean and distributional regression methods. 


Our method includes a semi-synthetic simulation framework using two domains, training ($A$) and testing ($B$), which have different covariate ($\bm{Z}$) and treatment ($X$) distributions, but whose \emph{conditional outcome distribution} (COD, $\Yx \mid \bm{Z}$) is assumed to be the same. First, we fit a frugally parameterized model \citep{evans2024parameterizing} to learn the COD $P_{\Yx|\bm{Z}}$ on domain $B$. The frugal parameterization allows us to obtain the \emph{marginal outcome distribution} (MOD) of $\Yx$ on domain $B$ 
explicitly as part of the joint. 
We then generate semi-synthetic outcome samples of domain $A$ by applying the COD of domain $B$, while using the covariates and treatments from domain $A$. 

Next, we train the causal model of interest on these semi-synthetic samples in domain $A$, and use it to estimate marginal causal quantities for domain $B$. The model's generalizability is assessed by statistically testing its ability to recover marginal causal quantities from domain $B$ against the \emph{explicitly known} ground truth inferred earlier. By reducing the complexity from higher-dimensional to a lower-dimensional causal effect, we simplify the evaluation process, enabling more powerful statistical testing.

The availability of exact marginal quantities in domain $B$ enables us to construct our proposed workflow.
In some real applications, it is usually the marginal quantities that are reported. For example, in many studies analyzing COVID-19 outcomes, researchers reported untreated outcomes, such as mortality rates or symptom progression, to contextualize treatment effects. The untreated mortality rate for severe COVID-19 in \cite{recovery2021dexamethasone} is often cited as a benchmark for evaluating interventions like dexamethasone. Our method thus  provides a simple and effective solution for assessing generalizability of an algorithm in complicated (real-world) data with statistical guarantees, including Type-I error control.

The code used for this paper can be found in \href{https://github.com/rje42/DomainChange}{\texttt{https://github.com/rje42/DomainChange}}.

\section{Background}\label{sec:background}
Consider a static treatment model with an outcome $Y \in \mathcal{Y}\subseteq \mathbb{R}$ and a general treatment $X$, which can be either continuous or discrete. In addition, we also make the standard causal assumptions of consistency, positivity, and conditional ignorability outlined in \citet{pearl2009causality} throughout the paper.
Let the set of measured pretreatment covariates be $\bm{Z} \in \mathcal{Z}\subseteq \mathbb{R}^{D}$. 
We then define the marginal \textit{causal} treatment density as
\begin{equation*}
    p_{\Yx}(y(x)) = \int p_{\YIZbX}(y \cmid \bm{z}, x) ~ p_{\bm{Z}}(\bm{z})~d\bm{z};
\end{equation*} 
this is marginalized over the covariates.  Here $\Yx$ is the \emph{potential outcome} for $Y$ given that $X$ is set to a value $x$.


We also use $\mu(x) = \mathbb{E}\, \Yx$ to denote the expected outcome given an intervention that sets $\{X=x\}$, and $\mu(x,z) = \mathbb{E}\left[Y(X=x)\mid Z=z\right]$ to denote the conditional expectation given covariate values. Note that $\Yx$ is essentially equivalent to $Y \mid \text{do}(X=x)$ in the notation of \citet{pearl2009causality}. When the treatment is binary, we define $\tau = \mathbb{E}[Y(1) -Y(0)]$ as the average treatment effect (ATE), quantifying the overall impact of a treatment change across the entire population. Similarly, let $\tau(z) = \mathbb{E}[Y(1) -Y(0)\mid Z=z]$ be the conditional average treatment effect (CATE), giving the result for specific subgroups or individuals, and therefore capturing treatment effect heterogeneity.

Denote the probability measures in domain A and domain B as $P^A$, $P^B$ respectively. Since our scenario requires that the conditional outcome distributions are the same we have $P^A_{\Yx\mid\bm{Z}}=P^B_{\Yx\mid\bm{Z}}$; however, since the covariate and treatment distributions may differ, the corresponding equality between the \emph{marginal} causal distributions does not necessarily hold. 

We aim to evaluate the generalizability of an outcome regression model $\hat{f}(\bm{z},x)$ that predicts the expected outcome $Y$. Predicted outcomes are denoted by $\hat{y} :=\hat{f}(\bm{z},x)$.

\subsection{Generalizability in Causal Inference}
\label{sec:generalizability_in_causal_inference}
Extensive research has focused on generalizability in causal inference, as mentioned in the introduction. 
As highlighted by \cite{ling2022critical}, three common approaches are used to assess treatment effect generalizability: inverse probability of sampling weighting (IPSW) methods that adjust for differences between study and target populations by weighting based on sample inclusion probabilities \citep{buchanan2018generalizing}; outcome models that estimate the conditional outcome directly \citep{kern2016assessing}; and hybrid approaches that combines both \citep{dahabreh2019generalizing}.

In this paper, we focus on algorithms that generalize conditional outcome predictions across different domains, enabling accurate CATE or COD estimation. This is crucial for understanding individual-level treatment effect heterogeneity and ensuring that models can adapt to new populations or environments with varying covariate distributions. A summary of common CATE estimation methods is provided by \cite{caron2022estimating}.


Despite advancements in CATE estimation, a systematic framework for evaluating generalizability remains underdeveloped. For example,  \citet{johansson2018learning} validate their model using both simulated and real world data. The simulated data examples assess predictive generalizability with MSE in the absence of any treatment mechanism, making causal verification impossible. Additionally, their analysis of the IHDP dataset \citep{hill2011bayesian} does not involve covariate or treatment shifts, so it does not effectively test generalizability. Another relevant paper is \citet{shi2021invariant}, which measures out-of-domain generalization performance using the mean absolute error (MAE). While their method achieves the lowest MAE among competitors,  there is no formal criterion to determine whether a specific MAE value signifies sufficient generalization to a new domain.

We highlight these issues not as criticisms of the papers, but to emphasize that robust generalizability evaluation methods of causal models are missing and challenging. Furthermore, existing benchmarks like IHDP are not specifically designed for out-of-domain generalization tests. To address this gap, we propose a systematic semi-synthetic framework to evaluate how well CATE algorithms perform across domains with different covariate distributions, offering a more practical assessment of whether a given approach will generalize well. In \Cref{sec:IHDP}, we adapt the IHDP experiments presented in \citet{johansson2018learning} and extend them by generating datasets from different domains, while making the marginal quantity explicitly known. Furthermore, we contrast the predictive MSE scores with the p-values derived from our tests to show how the latter provides a more actionable metric for whether a model successfully generalizes or not.

\subsection{Frugal Parameterization}\label{subsec:frugal-params}
A frugal parameterization of an observational joint distribution, $P_{\ZbXY}$, factorizes the distribution into a set of causally relevant components~\citep{evans2024parameterizing}. This decomposition explicitly parameterizes the marginal causal distribution, $P_{\Yx}$, or other lower dimensional causal distribution $P_{\Yx|\bm W}, \bm W\subset \bm Z$, and builds the rest of the model around it. Frugal models require that the three usual assumptions for causal inference (consistency, positivity, no unmeasured confounding) in addition to any additional regularity assumptions (further details can be found in Appendix A of \citet{evans2024parameterizing}).

Let us start by first parameterizing the \textit{conditional outcome distribution} (COD), $P_{\YxIZb}$. Frugal models can parameterize the COD in terms of the marginal causal distribution, $P_{\Yx}$, and a conditional copula distribution, $C_{\YxIZb}$. Here, $C_{\YxIZb}$ models the joint dependency between the marginal causal distribution and each of the univariate marginal covariate distributions, $\{P_{Z_i}\}_{i}$ such that
\begin{equation*}
    p_{\YxIZb} = p_{\Yx} \cdot  c_{\YxIZb},
\end{equation*}
where $c_{\YxIZb}$ is  a copula density function that parameterizes the dependence between $\Yx$ and the covariates. Multivariate copulas, particularly when parameterized using pair copula constructions or vine copulas~\citep{czado2022vine}, offer a rich flexible framework for modeling complex multivariate distributions, whilst also capturing (or allowing the user) to encode specific dependency constraints in the target data generating process. See \Cref{app:copulas} for further details on copulas and how they can be be fitted to real-world datasets.

This leaves the distribution of the \textit{past}, 
i.e.~the covariate distribution and the propensity score. We assume that all covariates are strictly pretreatment, so $\bm{Z}$ does not include any mediators of the causal effect of $X$ on $Y$. If we use a conditional copula then the past and the COD are variation independent, in the sense that they parameterize separate, non-overlapping aspects of the joint distribution.
This allows the past to be freely specified without affecting either the conditional copula, or the marginal causal distribution. 

The frugal parameterization also allows us to chose a conditional estimand.  For example, if we were interested in a conditional average treatment effect given $\bm{W} \subset \bm{Z}$, we could write $p_{\Yx|\bm{Z}} = p_{\Yx|\bm{W}} \cdot c_{\Yx|\overline{\bm{Z}}; \bm{W}}$ where $\overline{\bm{Z}} = \bm{Z}\setminus\bm{W}$.  Here $c_{\Yx|\overline{\bm{Z}}; \bm{W}}$ is a pair-copula between $\Yx$ and $\overline{\bm{Z}}$ conditional upon $\bm{W}$. This enables us to condition on a small subset of covariates that we consider to be particularly important in terms of predicting the outcome.

\section{Method}\label{sec:method}
\Cref{fig:algo-workflow} provides an overview of our workflow. We begin by defining both a test and a training domain, each with a distribution over the pretreatment covariates and the treatment, allowing for distribution shifts across covariates and treatment allocation. The COD is frugally parameterized with a conditional copula, where the covariates' cumulative distribution functions (CDFs) are derived from the test domain’s covariate densities. This ensures that samples from the test dataset follow a \textbf{known, customizable} marginal causal density, $p_{\Yx}$.

The training data is generated from the same COD, though since the covariate densities may not match the CDFs used to parameterize the conditional copula we do not have access to the marginal causal distribution in closed-form. We then learn a model, $\hat{f}(\bm{z},x)$, on the training data. Finally, a statistical test is performed to validate whether the lower-dimensional marginal quantity (such as the ATE or an expected potential outcome)  estimated using model outcomes equals the ground truth in the test domain.

\begin{figure}[h]
\vspace{.3in}
\centerline{\includegraphics[width=1\linewidth]{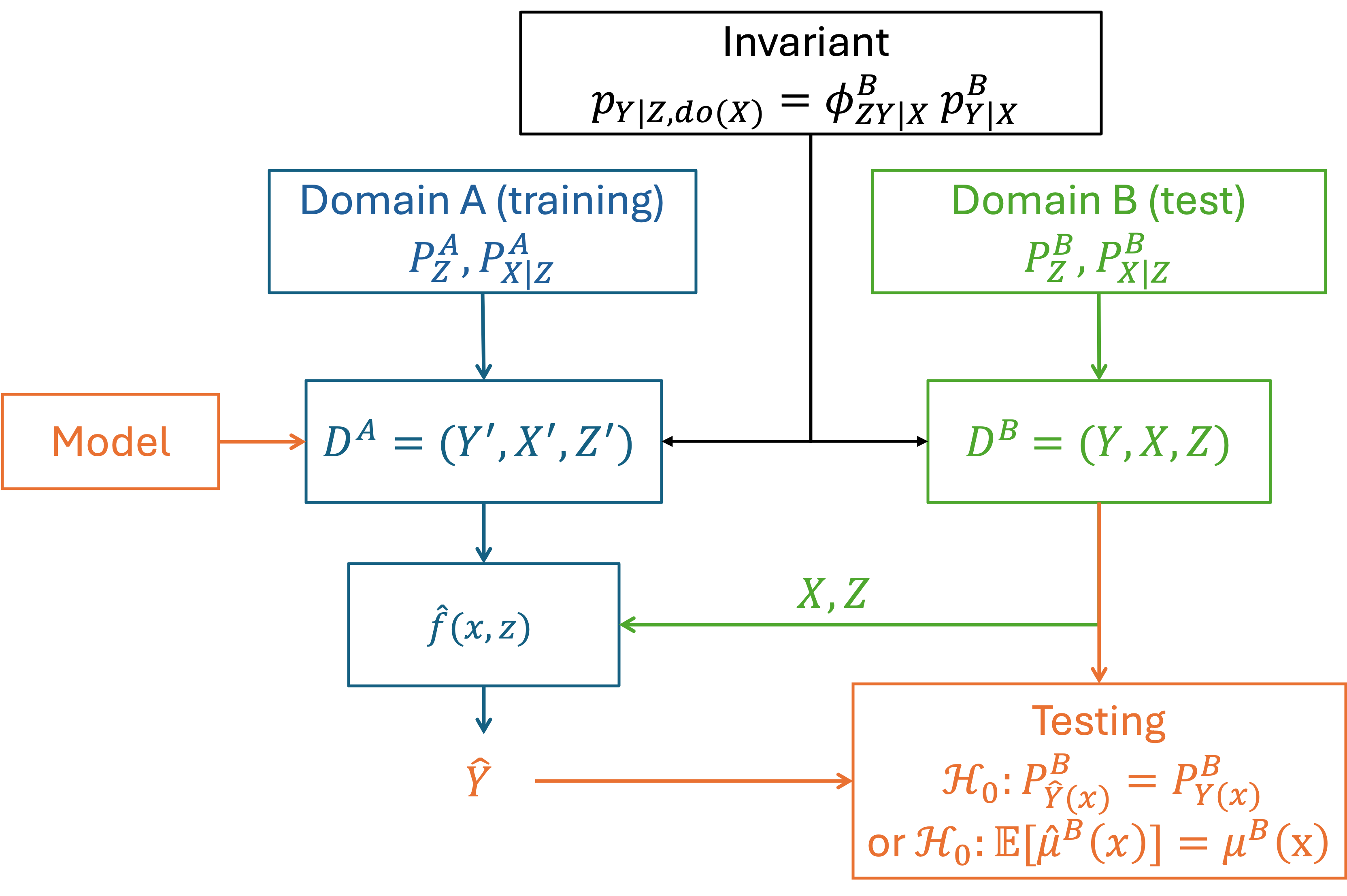}}
\vspace{.3in}
\caption{Workflow of the Proposed Method.}
\label{fig:algo-workflow}
\end{figure}

\subsection{Data Simulation}
In this section we describe how to simulate the data.


\subsubsection{Multi-domain Simulation with Frugal Models}
We begin by specifying two data generating processes: the training data, $D^{A} \sim P^{A}_{\ZbXY}$, and the test data, $D^{B} \sim P^{B}_{\ZbXY}$. Our goal is to construct a COD that parameterizes the joint density across both domains, while ensuring that the marginal causal density in domain $B$ is parameterized by $p^{B}_{\Yx}$. The supports of covariates in domains $A$ and $B$ are denoted $\mathcal{Z}^A$, $\mathcal{Z}^B$.

Recall from \Cref{subsec:frugal-params} that a general observational density can be factorized into the \textit{past}, $p_{\bm{Z}X}$, and the COD:
\begin{equation}\label{eq:cod}
    \begin{aligned}
        p&_{\YxIZb}(y \cmid \bm{z}) = p_{\Yx}(y) \times \\ 
        & \qquad c_{\YxIZb}\!\left(F_{\Yx}(y) \mid F_{Z_{1}}(z_1),\dots, F_{Z_{D}}(z_{d}) \right),
    \end{aligned}
\end{equation}
where $F_{\Yx}$ is the CDF associated with the marginal causal density $p_{\Yx}$. 

Note that the copula density in (\ref{eq:cod}) is not only determined by the copula's family and its parameterization, but also by the choice of marginal CDFs for the covariates, $\bm{Z}$. If the conditional copula density is marginalized over the densities corresponding to the covariate CDFs, then the ranks of the marginal causal density will be uniformly distributed:
\begin{equation*}
    p\left(F_{\Yx}\right) = \int d\bm{z}~c_{\YxIZb}(y(x) \cmid \bm{z}) \cdot \prod_{d=1}^{D}p_{Z_{d}}(z_{d}) = 1.
\end{equation*}
This uniformity is guaranteed if the marginal covariate densities $\{ p_{Z_d} \}_{d=1}^{D}$ correspond to the CDFs used to parameterize the copula. Thus, data simulated using our method matches the marginal causal quantity we specify. 


For evaluating generalization, we set the CDFs within the copula density to be derived from the covariate densities in the test domain $P_{\bm{Z}XY}^{B}$. This allows us to construct the COD density across all covariate and treatment spaces:
\begin{equation*}
    \begin{aligned}
        p&_{\YxIZb}\left( y \cmid \bm{z}\right) = p^{B}_{\Yx}\left(y\right) \times \\ 
        & \qquad c^{B}_{\YxIZb}\!\left(F^{B}_{\Yx}(y) \,\middle|\, F_{Z_1^{B}}(z_1), \dots, F_{Z_{D}^{B}}(z_D) \right),
    \end{aligned}
\end{equation*}
which will sample from a known marginal causal density equal to $p^{B}_{\Yx}$ if the covariate CDFs in the copula are derived from the test domain covariate densities. 

For two joint distributions with the same marginal covariate densities but different marginal causal densities, their CODs must differ. We can thus evaluate differences between CODs via comparing the lower-dimensional marginal causal densities instead.

This offers a great deal of flexibility in testing method generalizability. One can draw training and test datasets with different covariate densities and propensity scores, while guaranteeing that the CODs remain consistent, and that the test data is drawn from a distribution with a marginal causal density parameterized by $p^{B}_{\Yx}$. However, we note that a key assumption of our testing framework is $\mathcal{Z^A}\subseteq \mathcal{Z^B}$, as evaluating $p_{\Yx|\bm{Z}}$ requires evaluation of all marginal covariate CDF defined on domain $B$.

\begin{algorithm}[h!]
\caption{Semi-synthetic Data Generation.}
\begin{algorithmic}
\vspace*{2pt}
\STATE{\textbf{Input}:~Original test data; original covariates and treatment from training data.}
\vspace*{2pt}
\STATE{\textbf{Parameter estimations on test domain $B$}} 

Estimate the joint covariate-treatment density, $\hat{p}^{B}_{\ZbX}$; marginal causal density, $\hat{p}^{B}_{\Yx}$; conditional copula, $\hat{C}^{B}_{\YxIZb}$.
\STATE{\textbf{Data simulation on domain $B$}}

Sample $(\bm{z}^{B}, x^{B}) \sim \hat{p}^{B}_{\ZbX}$;\\
Sample the causal effect rank $\hat{u}^{B}_{\Yx|\bm{Z}} \sim U[0,1]$;\\
Calculate $y^{B} = {\big(\hat{F}_{\Yx}^{B}\big)^{-1}}\left(\hat{C}^{B}_{\Yx|\bm{Z}}(\hat{u}^{B}_{\Yx|\bm{Z}} \mid \bm{z}^{B})\right)$.

\vspace*{2pt}
\STATE{\textbf{Parameter estimation on training domain $A$}}

Estimate the joint covariate-treatment density, $\hat{p}^{A}_{\bm{Z}X}$.

\vspace*{2pt}
\STATE{\textbf{Data simulation on domain $A$}} 

Sample $(\bm{z}^{A}, x^{A}) \sim p^{A}_{\ZbX}$;\\
Sample the causal effect rank $\hat{u}_{\Yx|\bm{Z}}^{A} \sim U[0,1]$;\\
Calculate $y^{A} = {\big(\hat{F}_{\Yx}^{B}\big)^{-1}}\left(\hat{C}^{B}_{\Yx|\bm{Z}}(\hat{u}_{\Yx|\bm{Z}}^{A} \mid \bm{z}^{A})\right)$.

\vspace*{2pt}
\STATE{\textbf{Output}: Training sample $D^{A} = (\bm{z}^{A}, x^{A}, y^{A})$};\\ \hspace*{40pt} Test sample $D^{B} = (\bm{z}^{B}, x^{B}, y^{B})$. 
\end{algorithmic}
\label{alg:semisynthetic_data}
\end{algorithm}

Our primary workflow follows the approach outlined in \Cref{alg:semisynthetic_data}. 
First, we estimate the joint covariate-treatment density of the test data, denoted as $\hat{p}^{B}_{\ZbX}$. We then estimate the marginal causal density $\hat{p}^{B}_{\Yx}$ and the conditional copula $\hat{c}^{B}_{\YxIZb}$, capturing the covariate-outcome dependency conditional on treatment. Given covariate and treatment samples, we can calculate the causal density rank, $\hat{u}_{\Yx}$ using the conditional copula. The outcome can be calculated using the inverse transform ${\big(\hat{F}_{\Yx}^{B}\big)^{-1}}$. For the training data, we follow a similar approach. 

\subsection{Statistical Testing}

Tests of hypotheses about high-dimensional objects have very little power if we wish to consider a wide range of alternatives.  The lower-dimensional objects can potentially increase the chance of rejection substantially if the null hypothesis fails to hold. Given that we know the marginal causal density parameterized by $p^B_{\Yx}$ from the frugal parameterization, we are able to develop statistical testing on  
$\mu^B(x)$ rather than $\mu^B(\bm{z}, x)$ for mean regression models, and $P^B_{\Yx}$ instead of $P^B_{\YxIZb}$ for distributional regression.

Our testing algorithms require some parameters: $N_{btp}$ as the number of bootstrap iterations, $N^{A}$ and $N^{B}$ as the number of samples simulated from training domain and test domain for each bootstrap iteration, respectively. We provide the mean regression test in \Cref{alg:mean_test_algo}, but our algorithm can be extended to distributional regression models: after applying $\hat{f}$ to $D^{B}_b$, for each $i$,  we sample $\{y^j_{ib}\}_{j=1}^{N_Y}$ from the predicted distribution, $\hat{P}_{Y(x_{ib})|z_{ib}}$, and estimate marginal causal distributions such as $\hat{P}^{B}_{Y\left(x^0\right)} :=\bigcup_{b=1}^{N_{btp}} \bigcup_{i=1}^{N^{B}} \bigcup_{j=1}^{N_Y} \left\{ y_{ib}^j \mid x_{ib} = x^0 \right\}$. We then conduct distribution tests, e.g.~the Kolmogorov-Smirnov test, for $\mathcal{H}_0:  \hat{P}^{B}_{Y(x^0)}=P^{B}_{Y(x^0)}$ and get the p-value.

Our testing algorithm is flexible in the choice of testing reference, e.g.~in \Cref{alg:mean_test_algo}, we can replace $\mu^{B}(x)$ with $\tau^{B}$ as the reference target when $X$ is binary, which is what we used in our experiments. The testing method used for distributional regression models can also be replaced by other statistical tests, such as the Maximum Mean Discrepancy Test \citep{gretton2012kernel} or the Cramér-von Mises Test \citep{anderson1962distribution}.



\begin{algorithm}[t]
\caption{Generalizability Evaluation on Mean Regression Models.}
\begin{algorithmic}
\vspace*{2pt}
\STATE{\textbf{Input}:~~~~$\Theta^{A}$: parameters for training domain,\\
\hspace*{34.3pt} $\Theta^{B}$: parameters for test domain,\\
\hspace*{34.3pt} $\mu^{B}(x^0)$: reference.}
\vspace*{3pt}
\FOR{$b=1, \ldots, {N_{btp}}$}
    \STATE{Draw $D_b^{A}:= \{(\bm{z}'_{ib}, x'_{ib}, y'_{ib})\}_{i=1}^{N^{A}} \sim P_{\Theta^{A}}$};
    \STATE{Fit the regression model, $\hat{f}$, on $D_b^{A}$};
    \STATE{Draw $D_b^{B}:= \{(\bm{z}_{ib}, x_{ib})\}_{i=1}^{N^{B}} \sim P_{\Theta^{B}}$};
    \STATE{Apply $\hat{f}$ on $D_b^{B}$ to get predictions $\{\hat{f}(\bm{z}_{ib}, x_{ib})\}_{i = 1}^{N^{B}}$}; 
    \STATE{Calculate 
    $$\hat{\mu}_b^{B}(x^0) = \frac{\sum_{i=1}^{N^{B}} \mathbb{1}\{x_{ib}=x^0\}\hat{f}(\bm{z}_{ib}, x_{ib})}{\sum_{i=1}^{N^{B}}\mathbb{1}\{x_{ib}=x^0\}}.$$}
\ENDFOR
\STATE{\textbf{end for}}
\vspace*{3pt}
\STATE{Get the p-value by conducting a t-test to compare the target parameter $\mu^{B}(x^0)$ and the distribution of $\{\hat{\mu}_b^{B}(x^0)\}_{b=1}^{N_{btp}}$}.
\STATE{\textbf{Return} $p$.}
\vspace*{3pt}
\end{algorithmic}
\label{alg:mean_test_algo}
\end{algorithm}

A summary of this workflow is presented in \Cref{fig:algo-workflow}.

\section{Experiments}\label{sec:experiments}

In this section, we use our workflow to evaluate the generalizability of a range of modern causal models.

As discussed in several review papers like \cite{curth2021really}, \cite{ling2022critical} and 
 \cite{kiriakidou2022evaluation}, methods such as Meta-Learners (e.g.~T- and S-learners) \citep{kunzel2019metalearners}, CausalForest \citep{wager2018estimation}, TARNet \citep{shalit2017estimating}, and BART \citep{chipman2010bart} are widely used for CATE estimation, each offering advantages in different scenarios. Our evaluation focuses on their performance under covariate distribution shifts, specifically examining the accuracy of their CATE estimations. Further details about these models can be found in \Cref{sec:models}.

Another interesting algorithm to be evaluated is engression, introduced in \cite{shen2023engression}. It approximates the conditional distribution using a pre-additive noise model. Targeting at a distributional regression, the model is capable of extrapolating to unseen or underrepresented data points through its learned non-linear transformations.  The key factors which affect engression's generalizability are the distances between two domains, and whether the true underlying function must be strictly monotonic in the extrapolation region. In our experiments, we evaluate engression in both the S-learner and T-learner settings.

\subsection{Synthetic Data}
\label{sec:synthetic}
We first conduct experiments on synthetic data to demonstrate and validate our method. While our approach can handle various data types and is particularly effective with high-dimensional covariates and continuous treatment interventions, for clarity, in this simple example, we focus on two continuous confounders, $Z_1$ and $Z_2$, sampled from identical gamma distributions, with a binary intervention $X$. We initially assume that both datasets come from  randomized controlled trials (RCT), so that $X \sim \operatorname{Bernoulli}(0.5)$ under $P^A$ and $P^B$.  We parameterize the Gaussian copula, $c_{\ZbYx}$, with Spearman correlation coefficients $\rho_{Z_1 Z_2} = 0$, $\rho_{Z_1\Yx} = 0.1$ and $\rho_{Z_2\Yx} = 0.9$. The distribution of $\Yx$  is defined as $\mathcal{N}(2x+1,1)$ in the test domain. For the simulation, we generate $N^{A} = 200$ training samples and 
$N^{B} = 50$ test samples per bootstrap, with $N_{btp}=200$ bootstraps in total, repeating this process for 50 iterations. The marginal distributions of $Z_1$ and $Z_2$ in the training domain follow identical Gamma distributions with shape $k=1$ and rate $\theta=1$.

We examine two settings: in Setting 1, the test domain has a slight covariate shift, with $Z_1$ and $Z_2$ following a Gamma distribution of $k=2$, $\theta=1$. In Setting 2, the shift is more significant ($k=4$, $\theta=1$). Despite these shifts, the COD remains the same due to frugal parameterization, as shown in \Cref{fig:synthetic}.

\begin{figure}[h]
\vspace{.3in}
\centerline{\includegraphics[width=1\linewidth]{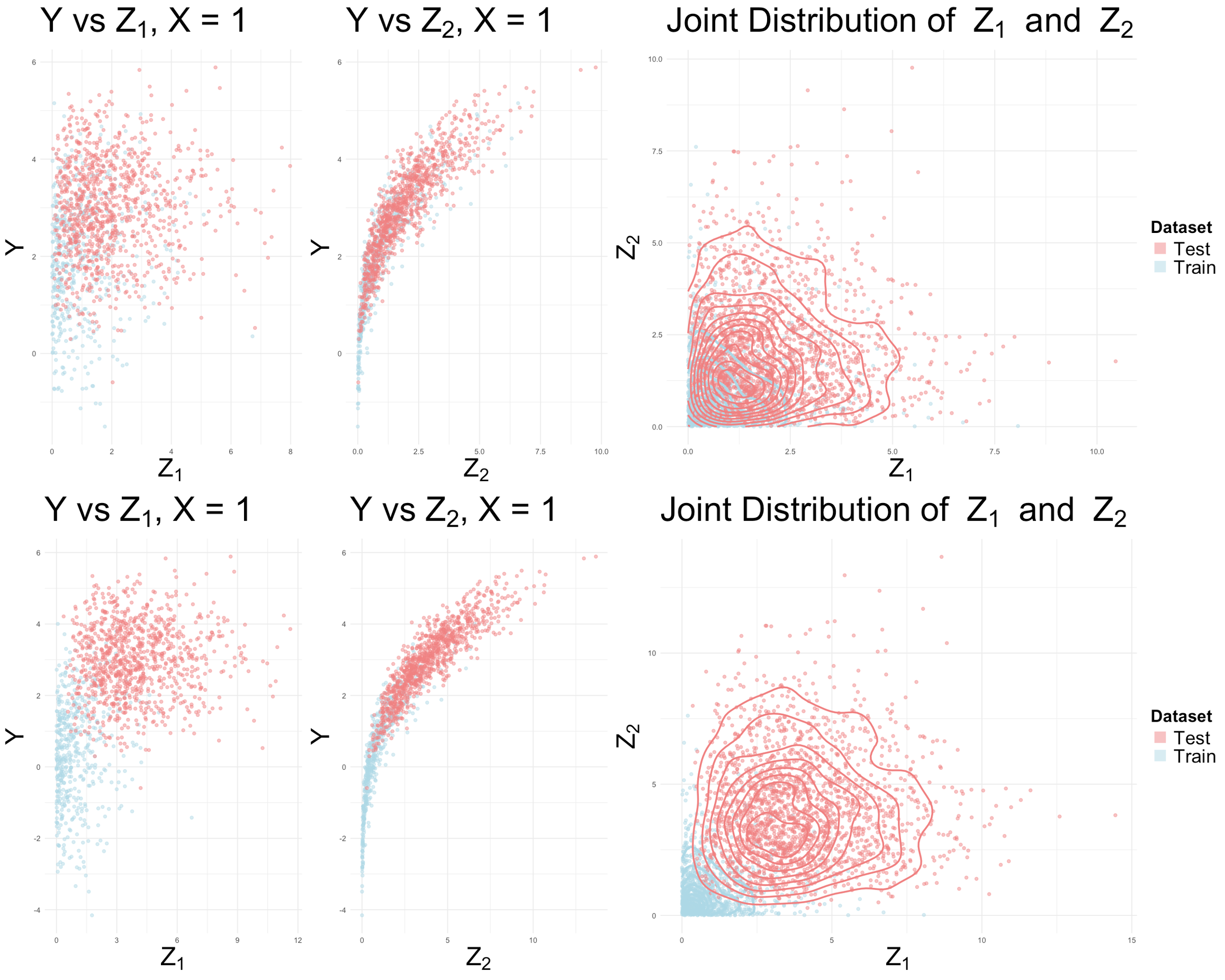}}
\vspace{.3in}
\caption{Synthetic Data Generated from Setting 1 (Top) and Setting 2 (Bottom). }
\label{fig:synthetic}
\end{figure}

The p-values in \Cref{fig:synthetic_mean_p} illustrate the differences across models. As expected, with a more significant domain shift in Setting 2, models face greater difficulty in generalizing, as reflected by the smaller p-values generally compared to Setting 1. T-BART and T-engression showed good generalizability performances in this specific setting with their p-values being uniformly distributed. TARNet struggles, likely due to the complexity of its representation learning network design and hyperparameter tuning.

\begin{figure}[h]
\vspace{.3in}
\centerline{\includegraphics[width=1\linewidth]{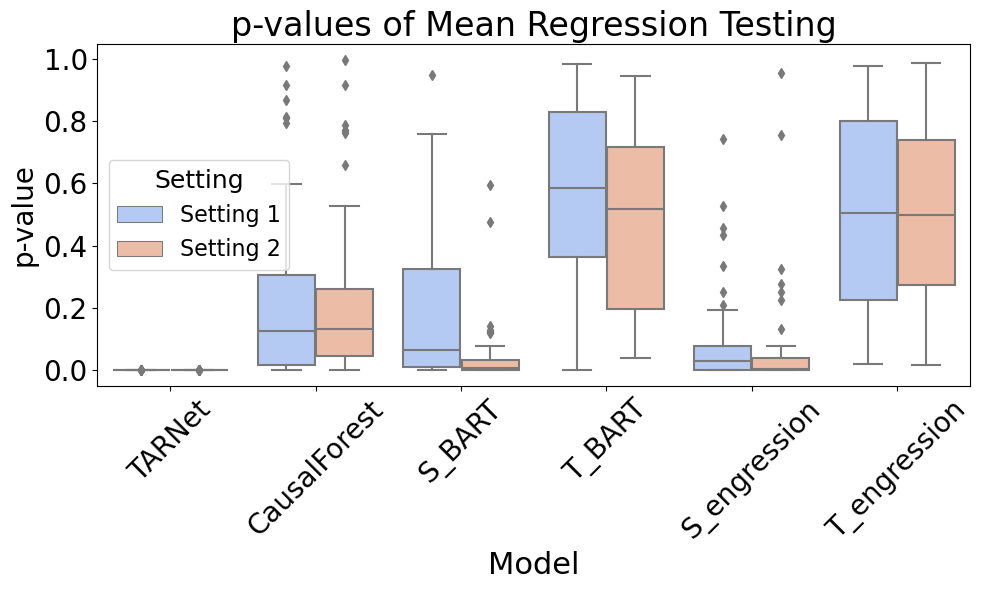}}
\vspace{.3in}
\caption{$p$-values of Mean Regression Testing, Synthetic Data of 50 Iterations.}
\label{fig:synthetic_mean_p}
\end{figure}

With our method, we are able to test the generalizability of distributional regression. \Cref{fig:synthetic_distribution_p} demonstrates the p-values of distributional regression testing of S-engression under the two settings, with $N_Y=50$. Not surprisingly, since the covariate distribution shift in Setting 1 is smaller, S-engression demonstrates better generalizability compared to that in Setting 2.

\begin{figure}[h]
\vspace{.3in}
\centerline{\includegraphics[width=1\linewidth]{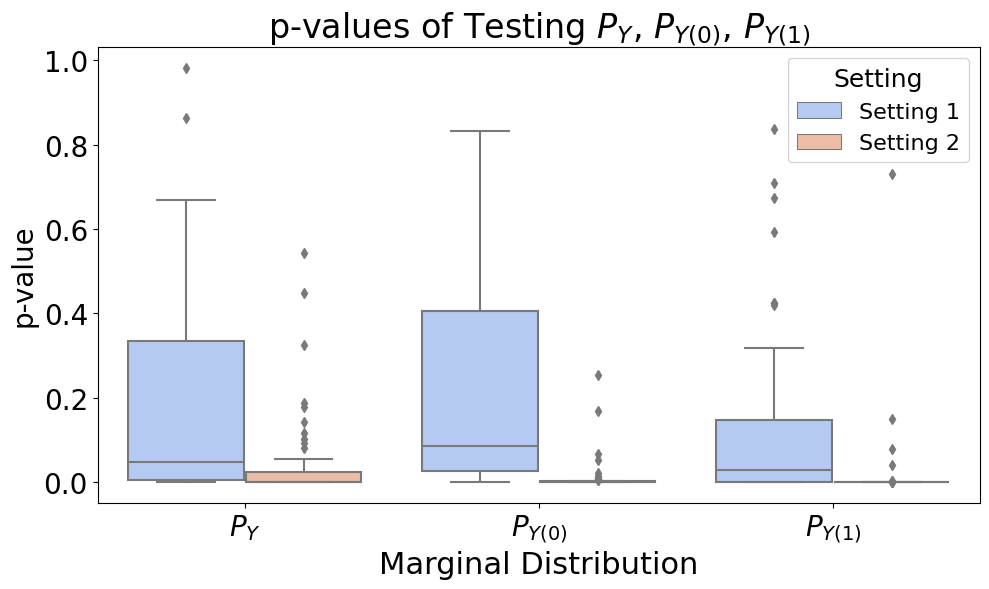}}
\vspace{.3in}
\caption{$p$-values of Distributional Regression Testing (Kolmogorov–Smirnov Test) of S-engression, Synthetic Data of 50 Iterations.}
\label{fig:synthetic_distribution_p}
\end{figure}

Supported by flexible simulations based on actual data, our method is useful for stress testing and model diagnostics. \Cref{fig:varying_n} shows how varying the training set size affects the generalizability of T-BART and T-engression; the performance worsen as $N^{A}$ exceeds 100. This issue may stem from problems like overfitting, but solving these problems is not our focus. Rather, our method serves as a tool to detect and highlight potential issues when making predictions on real data, which is feasible with the simulation based on actual data using the frugal parameterization. We also wish to remark on the difference between the performances of S- and T-learners. In CATE estimation, T-learners fit separate models for each treatment group while S-learners fit a single model across both, with treatment included as a feature. Hence, T-learners offer greater flexibility for modeling patient heterogeneities and it is unsurprising that they consistently outperform S-learners in our experiments.

\begin{figure}[t]
\vspace{.3in}
\centerline{\includegraphics[width=1\linewidth]{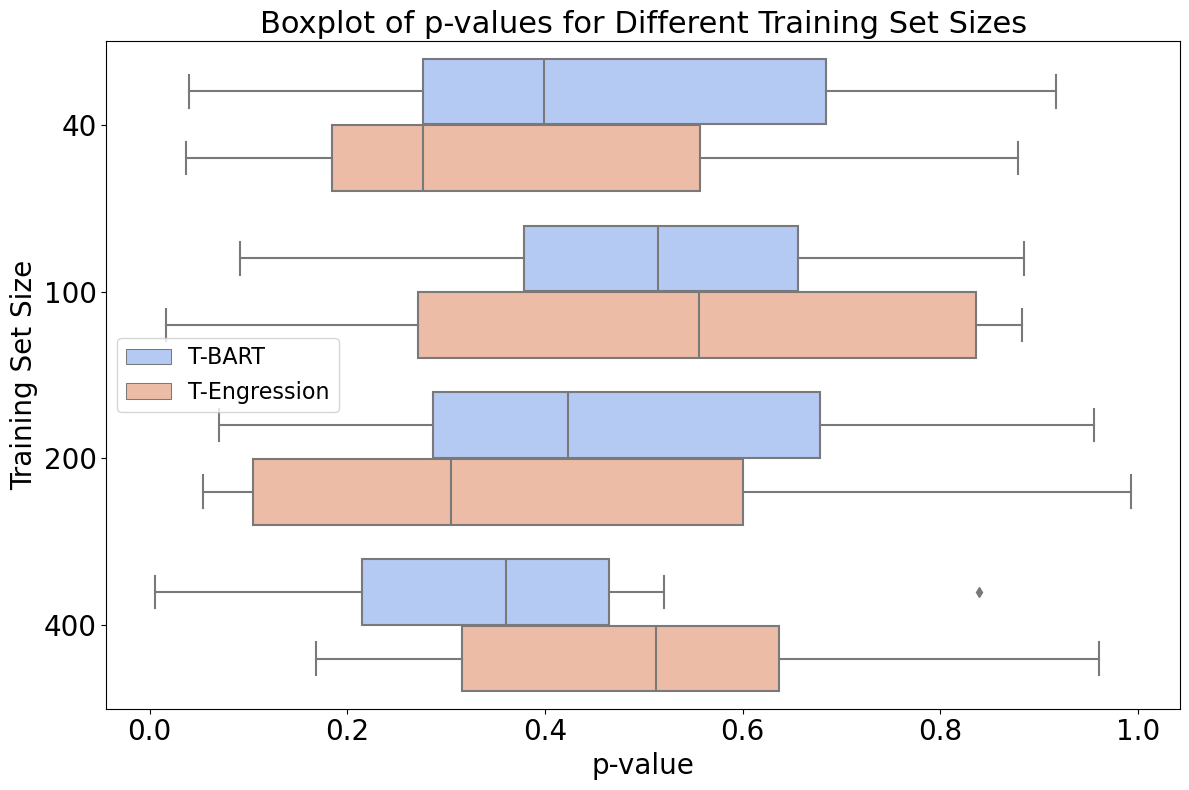}}
\vspace{.3in}
\caption{$p$-values of Mean Regression Testing of 50 Iterations, Varying $N^{A}$, Setting 2, Synthetic Data.}
\label{fig:varying_n}
\end{figure}

Note that extrapolation performance for models like engression is typically evaluated visually, one dimension at a time. Our method, however, offers significant advantages by providing statistical evaluation of extrapolation performance in high-dimensional covariates.


\subsection{Real Data}
\label{sec:IHDP}
We evaluate algorithm generalizability using the Infant Health and Development Program (IHDP) dataset, a randomized experiment conducted between 1985 and 1988 to study the effect of home visits on infants' cognitive test scores~\citep{hill2011bayesian}. This dataset has become widely used in domain adaptation research \citep{curth2021really,shi2021invariant}. 

In this section, we extend the experiments presented by \citet{johansson2018learning} which train a range of causal ML algorithms on IHDP data and measure in-domain predictive performance using MSE. We extend these experiments by showing how our validation framework can be used to test out-of-domain predictive performance. Specifically, we compare the MSE metric against the p-values obtained via our proposed testing framework, highlighting how our method provides a more informative metric of whether a model can generalize robustly across different domains.

The IHDP dataset contains $T=1000$ trials, each consisting of the same 747 subjects and 25 pretreatment covariates, with the first six being continuous and the rest binary.  The potential outcomes $Y\mspace{-1mu}(1)$ and $Y\mspace{-1mu}(0)$ are provided in the data. In each trial $Y\mspace{-1mu}(x) \sim \mathcal{N}(\bm{Z}\beta_t + 4t, \, 1)$, and $\beta_t$ is randomly chosen from values $(0, 1, 2, 3, 4)$ with probabilities $(0.5, 0.2, 0.15, 0.1,0.05)$. Thus, the potential outcomes vary across trials, while the covariates, CATE and ATE remain constant.

First we treat both domains as RCTs, that is, setting the propensity score model as $X\sim \operatorname{Bernoulli}(0.5)$ for all units. The observed outcome is then $Y = X Y\mspace{-1mu}(1) + (1-X) Y\mspace{-1mu}(0)$ by consistency. We randomly select 50 trials from the 1000 available, with each trial used to create one training-test pair, and evaluate the model's generalizability on them. To introduce domain shift, we keep all covariate values identical between the training and test domains, except for $Z_1$, which is set to 1.5 times the original value in the test domain compared to the training domain. For each training-test pair, we learn the parameters following \Cref{alg:semisynthetic_data}, specifying the marginal causal distribution to follow a Gamma distribution  with its parameters estimated from the IHDP data by fitting a generalized linear model. We denote the resulting data generation distributions as $P_{\Theta^{A}}, P_{\Theta^{B}}$ for the training and test domains, respectively. We sample training data of $N^{A} = 1000$ from $P_{\Theta^{A}}$, and $N^{B} = 200$ test data from $P_{\Theta^{B}}$. The number of bootstraps is set to be $N_{btp} = 200$. Note that in our experiments, the outcomes were shifted to ensure they are strictly positive, allowing us to use the parametric form of the Gamma distribution to obtain an explicit expression for the mean.

\Cref{fig:ihdp_mean} shows the boxplot of the $\log_{10}$($p$-values) of each model and \Cref{tab:ihdp_percentage} contains the percentage of $p$-values greater than 0.05 across the $50$ trials.  T-/S-engression demonstrate better generalizability in this setting among all these methods.  We also give the result of distributional regression testing in \Cref{fig:ihdp_dist}.

\begin{figure}[h]
\vspace{.3in}
\centerline{\includegraphics[width=1\linewidth]{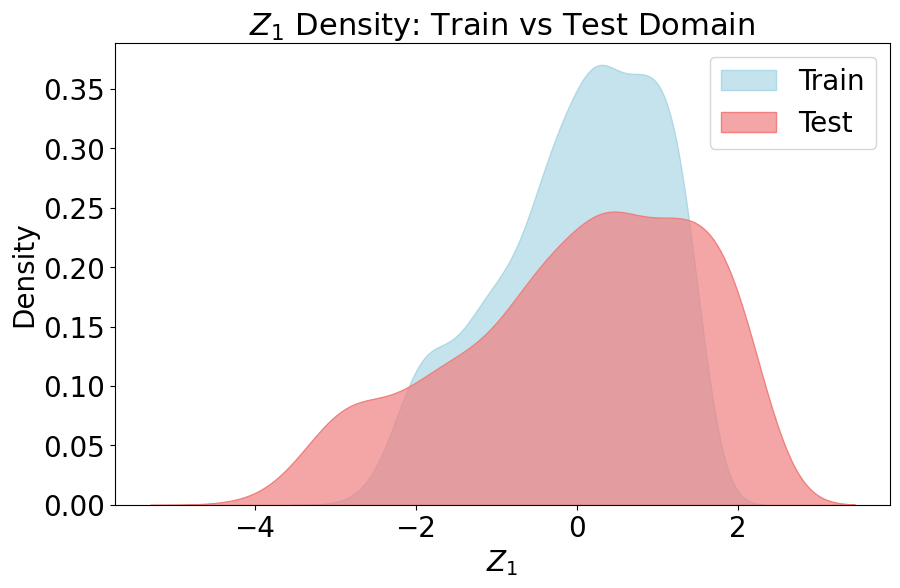}}
\vspace{.3in}
\caption{Density of $Z_1$ of Training and Test Domains.}
\label{fig:ihdp_shift}
\end{figure}
\begin{figure}[t]
\vspace{.3in}
\centerline{\includegraphics[width=1\linewidth]{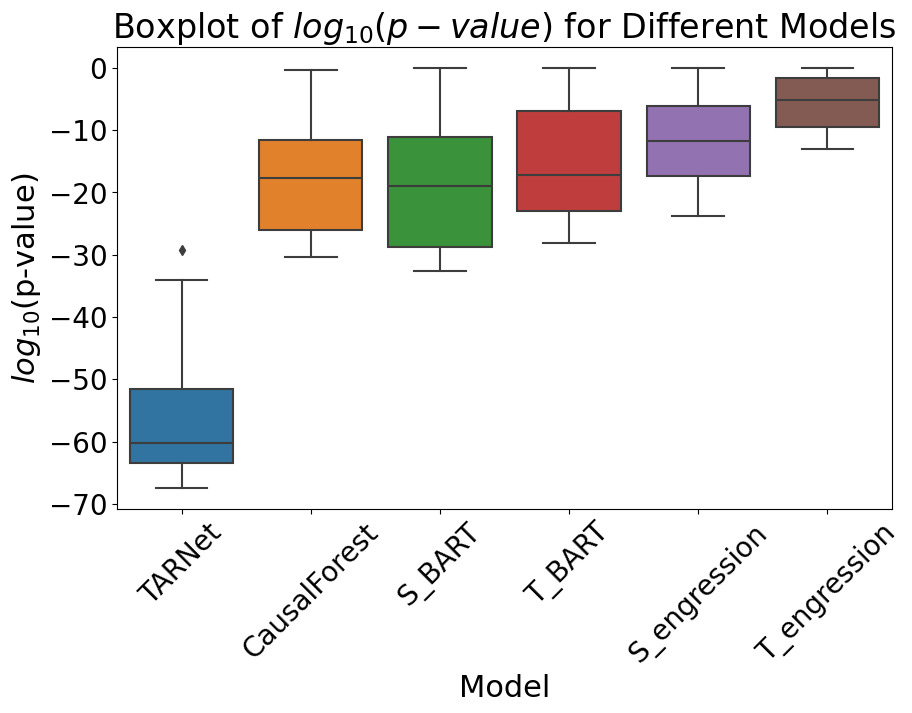}}
\vspace{.3in}
\caption{$\operatorname{log}_{10}(p\text{-values})$ of Mean Regression Testing of 50 Trials in IHDP.}
\label{fig:ihdp_mean}
\end{figure}

\begin{table}[h]
\begin{center}
\begin{tabular}{rrr}
\toprule
\textbf{Model} & \textbf{RCT} & \textbf{Non-RCT} \\
\midrule
TARNet & 0\% & 0\% \\

CausalForest & 12\% & 6\%\\

S-BART & 12\% & 8\% \\

T-BART & 12\% & 6\% \\

S-engression & 18\% & 6\%\\
T-engression & 24\% & 8\%\\
\bottomrule
\end{tabular}
\end{center}

\caption{Percentage of $p > 0.05$ across 50 Trials.} 
\label{tab:ihdp_percentage}

\end{table}

\begin{figure}[t]
\vspace{.3in}
\centerline{\includegraphics[width=1\linewidth]{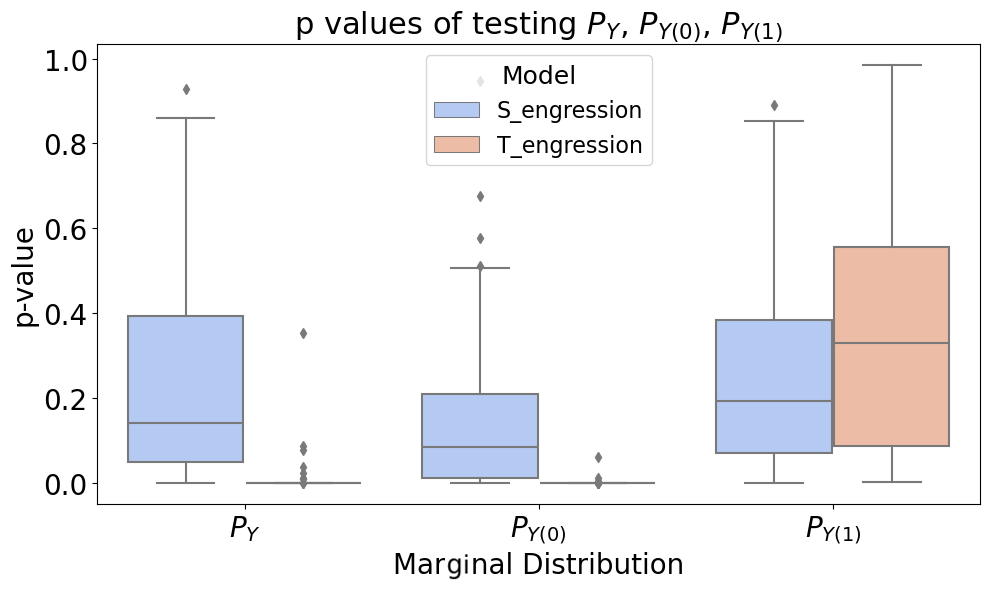}}
\vspace{.3in}
\caption{$p$-values of Distributional Regression Testing of 50 Trials in IHDP.}
\label{fig:ihdp_dist}
\end{figure}


While we use the RCT setting as an example above to demonstrate our method, it is also applicable to observational studies. In a non-randomized setting where treatment arms are imbalanced by setting $P(X=1 \mid Z) = \operatorname{logit}(Z_2+Z_3+Z_4)$,  the percentage of $p>0.05$ across 50 trials of each algorithm is shown in \Cref{tab:ihdp_percentage}. Since our paper's focus is on providing a systematic generalizability evaluation method, we omit further analysis here. 

Although we present such percentage, all p-values, including their distributions, are highly informative. We provide guidance of interpreting the testing results in \Cref{sec:read_p}. 

Note that this framework of constructing statistical tests on marginal quantities is not restricted to out-of-domain generalization testing. We adapt the original experiments in \citet{johansson2018learning}, in which the in-domain model performance was evaluated on IHDP data, and show how our framework can be easily adapted to performance evaluation for in-domain tasks. Since our method was designed for out-of-domain generalizability assessment, we do not discuss this further and leave a detailed discussion in \Cref{sec:indomain}.

Details on hyperparameters and additional experiments, including performance comparisons with or without domain shift when the CATE is known to be linear, are provided in \Cref{sec:computation_details} and \Cref{sec:linear}. 

\section{Discussion}\label{sec:discussion}

We make a few discussion remarks in this section.

\paragraph{Flexibility of vine copula specification} In \Cref{sec:experiments}, we present both fully simulated examples and a semi-synthetic experiments based on the IHDP dataset commonly used to validate generalization in causal inference tasks. We experimented with only Gaussian copulas with a fully connected dependency structure on a relatively small number of covariates. However, our framework can be extended to high-dimensional covariates settings and more complex dependency structures. For example, pair-copula constructions allow for flexible modeling  of non-Gaussian copulas with complex dependency structures. Additionally, one may consider using models such as Frugal Flows \citep{de2024marginal} which fit a more flexible, non-parametric generative frugal model to real world data. Further details and experiments for each of these cases can be found in \Cref{app:vinecop} and \Cref{sec:complicated_exp}. Although our approach is mainly designed for evaluation, we provide additional experiments addressing capability of our method handling model misspecification when generating semi-synthetic data (see the end of \Cref{sec:complicated_exp}).


\paragraph{Equivalence testing} While our approach of rejecting the null hypothesis shows that a model is not generalizable, it does not quantify the extent of failure. An extension of this approach may be to develop a more flexible testing method, inspired by equivalence testing \citep{wellek2002testing}. This would assess not just whether a model fails but also by how much, determining if its performance is significantly worse than a given threshold, offering a more nuanced view than traditional hypothesis testing. We provides some results in \Cref{sec:equiv_testing}. In this paper, we only consider marginal causal quantities as the validation references, but our framework can be easily adapted to use low-dimensional CODs as the reference instead with the flexibility of frugal parameterization (see \Cref{subsec:frugal-params}).

\paragraph{Validity of using low-dimensional proxy} We would also like to emphasize the objective of our method, which is to test the quality of fit of a conditional quantity against a lower dimensional marginal target instead. By introducing a low-dimensional proxy that would be identifiable under the true model, we aim to provide a quantity that is more tractable for testing, even if it sacrifices identification of a unique, correct CATE. While different CATEs can lead to the same marginal outcomes, we argue that this degeneracy is not a critical limitation in our setting. A lack of rejection simply indicates insufficient evidence that the model fails to generalize, instead of guaranteeing correctness of the CATE. We recognize that a model could fit an incorrect CATE while still producing accurate marginal outcomes. However, our empirical results suggest that such cases are rare in practice. Here, we appeal to a general result of the following form: the set of distributions where the COD fails to generalize but the marginal estimand does is a measure zero subset of distributions where the COD fails. This is analogous to the so-called ‘faithfulness’ argument for causal discovery algorithms \citep{spirtes2000causation}, or the `completeness' of d-separation. In finite samples we would need a stronger assumption (more analogous to `strong faithfulness' in \citealp{zhang2002strong}) to avoid such false negatives. This is beyond the scope of our paper. 


\section{Summary}\label{sec:summary}

In this paper, we develop a statistical method for evaluating the generalizability of causal inference algorithms using actual application data, facilitated by the frugal parameterization. Our approach introduces a semi-synthetic simulation framework that bridges the gap between synthetic simulations and real-world applications, supporting the generalizability evaluation of both mean and distributional regression models. With flexible, user-defined data generation processes, our framework provides a principled, binary decision about whether or not a model is generalizable to a specific domain. This is essential for model selection. In practice, our method helps structure the selection process into two stages:
\begin{itemize}
    \item Stage 1: Apply the proposed testing procedure to identify models that generalize across domains.
    \item Stage 2: Among the models that pass the test, use a metric like MSE to choose the best-performing one.
\end{itemize}

This two-stage approach ensures that model selection is both statistically sound and practically robust, as it prioritizes generalizability before performance evaluation. Following this framework, we select models that are ``good and generalizable'', rather than just ``relatively good'' without generalizability assessment via MSE alone. We provide more details of the comparison between our method and MSE in \Cref{sec:compare_with_MSE}.

Through experiments on the synthetic and IHDP datasets, we assess the generalizability of algorithms such as TARNet, CausalForest, S-/T-BART, and S-/T-engression under domain shift. Our method acts as a valuable diagnostic tool, allowing us to explore how factors like training set size or covariate shifts impact generalizability. These insights can help identify model strengths and weaknesses and inform how causal inference models adapt to different settings.

We hope that this work inspires a more careful consideration of model evaluation, encourages simulations that better reflect real-world conditions, and highlights the importance of stress testing in advancing causal inference methodologies.

\begin{acknowledgements} 
The authors would like to thank Laura Battaglia and Xing Liu for their helpful comments on the paper.

D.d.V.M is supported by a studentship from the UK’s EPSRC's Doctoral Training Partnership (EP/T517811/1). L.Y. is supported by the EPSRC Centre for Doctoral Training in Modern Statistics and Statistical Machine Learning (EP/S023151/1) and Novartis. 
\end{acknowledgements}

\newpage
\bibliography{uai2025-template}

\newpage

\onecolumn

\title{Supplementary Material}
\maketitle
\appendix
\section{COPULA BACKGROUND}\label{app:copulas}

Copulas provide a powerful tool to model joint dependencies, independent of the univariate margins. This aligns well with the requirements of the frugal parameterization, where dependencies need to be varied without altering specified margins (the most critical being the specified causal effect). Understanding the constraints and limitations of copula models ensures that causal models remain accurate and consistent with the intended parameterization.

\subsection{SKLAR'S THEOREM}
Sklar's theorem \citep{sklar1959,czado2019analyzing} provides the fundamental foundation for copula modelling by providing a bridge between multivariate joint distributions and their univariate margins. It allows one to separate the marginal behaviour of each variable from their joint dependence structure, with the latter being the copula itself.

\begin{theorem}
For a d-variate distribution function $F_{1:d} \in \mathcal{F}(F_1,\ldots,F_d)$, with $j$th univariate margin $F_j$, the copula associated with $F$ is a distribution function $C : [0,1]^d \rightarrow[0,1]$ with uniform margins on $(0,1)$ that satisfies
\begin{equation*}
    F_{1:d}(\bm{y}) = C(F_1(y_1),\dots,F_{d}(y_d)), \qquad \bm{y} \in \mathbb{R}^{d}.
\end{equation*}
\begin{enumerate}
    \item If $F$ is a continuous $d$-variate distribution function with univariate margins $F_1,\dots, F_d$ and rank functions $F^{-1}_1,\dots, F^{-1}_d$ then
    \begin{equation*}
        C(\bm{u}) = F_{1:d}(F^{-1}_1(u_1),\dots,F^{-1}_d(u_d)), \qquad \bm{u}\in[0,1]^d.
    \end{equation*}
    \item If $F_{1:d}$ is a $d$-variate distribution function of discrete random variables (more generally, partly continuous and partly discrete), then the copula is unique only on the set
    \begin{equation*}
        Range(F_1) \times \dots \times Range(F_d).
    \end{equation*}
\end{enumerate}
The copula distribution is associated with its density $c(\cdot)$,
\begin{equation*}
    f(\bm{y}) = c(F_1(y_1),\dots, F_d(y_d))\cdot f_1(y_1)\cdots f_d(y_d),
\end{equation*}
where $f_i(\cdot)$ is the univariate density function of $Y_i$. 
\end{theorem}

Note that Sklar's theorem explicitly refers to the \textbf{univariate marginals} of the variable set $\{Y_1,\dots, Y_d\}$ to convert between the joint of univariate margins $C(\bm{u})$ and the original distribution $F(\bm{y})$. For absolutely continuous random variables, the copula function $C$ is unique. This uniqueness no longer holds for discrete variables, but this does not severely limit the applicability of copulas to simulating from discrete distributions.

An equivalent definition (from an analytical purview) is $C: [0, 1]^d \rightarrow [0, 1]$ is a $d$-dimensional copula if it has the following properties: 
\begin{enumerate}
    \item $C(u_1,\dots, 0, \dots, u_d) = 0$;
    \item $C(1, \dots, 1, u_i, 1, \dots, 1) = u_i$;
    \item $C$ is $d$-non-decreasing.
\end{enumerate}
\begin{definition}
    A copula $C$ is $d$-non-decreasing if, for any hyper-rectangle $H=\prod_{i=1}^{d}\left[u_i, y_i \right]\subseteq [0,1]^{d}$, the $C$-volume of $H$ is non-negative
    \begin{equation*}
        \int_{H}C(\bm{u})~d\bm{u} \geq 0.
    \end{equation*}
\end{definition}

\subsection{COPULAS FOR DISCRETE VARIABLES}\label{appsub:discrete-copulas}

Modelling the dependency between discrete and mixed data is particularly challenging, as copulas for discrete variables are not unique. Additionally, copulas encode ordering in the joint, and hence should only be used for count or ordinal data models. 
In order to deal with discrete variables, we use a the Generalized Distributional Transform of a random variable found originally proposed by \citet{ruschendorf2009distributional}.

\begin{theorem}
On a probability space $(\Omega, \mathcal{A}, P)$ let $X$ be a real random variable with distribution function $F$ and let $V \sim U(0, 1)$ be uniformly distributed on $(0, 1)$ and independent of $X$. The \textit{modified distribution function} $F(x, \lambda)$ is defined by
\begin{equation*}
F(x, \lambda) := P(X < x) + \lambda P(X = x).
\end{equation*}
We define the (generalized) \textit{distributional transform} of $X$ by
\begin{equation*}
U := F(X, V).
\end{equation*}
An equivalent representation of the distributional transform is
\begin{equation*}
U = F(X-) + V(F(X) - F(X-)).
\end{equation*}
\end{theorem}

\citet{ruschendorf2009distributional} makes a key remark about the generalized transform's lack of uniqueness for discrete variables. 

\subsection{PAIR COPULA CONSTRUCTIONS AND VINE COPULAS}\label{app:vinecop}
Pair copula constructions (PCCs) provide a flexible framework for modelling multivariate dependence by decomposing a high-dimensional copula into a sequence of bivariate copulas~\citep{bedford2002}. A vine copula is a specific class of PCCs that employs a graphical model to structure these pairwise dependencies, extending traditional copulas to describe complex dependency structures in high-dimensional data. 

Vine copulas allow for flexible modelling of more complex conditional dependence structures, enabling a richer representation of statistical relationships. This flexibility makes vine copulas particularly useful when modelling more complex multivariate distributions where different pairwise interaction types and conditional dependencies must be specified~\citep{czado2022vine,czado2019analyzing}.
Vine copulas extend this concept by decomposing a multivariate copula into a sequence of bivariate copulas arranged in a hierarchical structure. This decomposition enables the flexible modelling of dependencies among variables while preserving computational tractability.

There is a vast literature in showing how vines can parametrize different dependency structures, and allow for more complex and richer dependencies to be expressed using different vine tree structures and choices of copula families for each of the bivariate copulas in the vine.

The hierarchical organization of dependencies in vine copulas is achieved through a sequence of trees $\{T_1, T_2, \dots\, T_{K}\}$. Each tree consists of nodes and edges that represent variables and their dependencies, respectively. The first tree $T_1$ defines the marginal pairwise dependencies between variables. Each subsequent tree $T_k$ defines the dependencies conditional on the edges of the previous tree $T_{k-1}$. Each edge in $T_k$ is associated with a bivariate copula that models the conditional dependency between two variables. Mathematically, the joint density defined over a set of $d$ marginally uniform random variables, $c(u_1, \dots, u_d)$ of a vine copula can be expressed as:
\begin{equation}
c(u_1, \dots, u_d) = \prod_{k=1}^{d-1} \prod_{(i,j) \in E_k} c_{ij|D_{ij}}(u_i, u_j | u_{D_{ij}}),
\end{equation}
where $E_k$ represents the edges in the $k$th tree, and $D_{ij}$ denotes the conditioning set for the pair $(i, j)$.

Vine copulas model complex dependencies by combining bivariate copulas---such as Gaussian, Clayton, Gumbel, or Frank---that capture various types of correlation, including tail dependence and forms of asymmetry. The tree structure defines the choice of the order of dependence, and parameters are estimated from empirical data or assumptions. Their main strength lies in decomposing high-dimensional problems into tractable lower-dimensional components, enabling efficient sampling and inference.  It does all this while preserving computational tractability. 
%
%
In our experimental framework, we leverage these properties to evaluate the impact of different dependency structures on causal inference generalizability.

\subsection{FITTING AND CUSTOMIZING FRUGAL COPULA FITS}\label{subapp:fitting-copulas}
Vine copulas allow for a great deal of flexibility for customizing complex variable dependency structures in addition to efficient method for fitting real world datasets.

\paragraph{High Dimensional Covariate Fits} For real world or semi-synthetic data examples, we recommend the use of vine copulas for higher-dimensional and more complex dependency structures. For model selection, a popular choice is the \emph{Dissmann algorithm}, which fits vine copulas iteratively from the lowest tree level upwards~\citep{dissmann2013selecting}; the choice of bivariate copula families can be performed afterwards. We use the implementation in rvinecopulib \citep{rvinecopulib2025}, which performs structure selection and optimal bivariate copula family fitting. For more flexible nonparametric alternatives, we also highlight frugal flows as a viable generative model for learning expressive causal marginals via normalizing flows, although its performance suffers in very high dimensional settings \citep{de2024marginal}.

\paragraph{Computational Efficiency of Vine Copula Fits} Fitting vine copula models is not computationally prohibitive, even in high dimensional covariate settings. To further aid reproducibility and assess feasibility, \Cref{tab:copula_fitting_times} presents the time taken to fit vine copulas (both structure and bivariate family selection) across different dimensions, using a dataset with $N=200$ and a MacBook Pro M1 Pro, 2023. The results are averaged over 10 different fits.
\begin{table}[h!]
\centering
\begin{tabular}{c r r r r r}
\hline
\multirow{2}{*}{\textbf{Sample Size}} & \multicolumn{5}{c}{$\boldsymbol{D}$}\\
& \multicolumn{1}{c}{$\mathbf{10}$} & \multicolumn{1}{c}{$\mathbf{25}$} & \multicolumn{1}{c}{$\mathbf{50}$} & \multicolumn{1}{c}{$\mathbf{100}$} & \multicolumn{1}{c}{$\mathbf{200}$} \\
\hline
10   & 0.13 $\pm$ 0.01 & 0.90 $\pm$ 0.25 & 3.5 $\pm$ 0.04 & 14 $\pm$ 0.07 & 61 $\pm$ 1.40 \\
25   & 0.21 $\pm$ 0.01 & 1.40 $\pm$ 0.04 & 5.7 $\pm$ 0.09 & 23 $\pm$ 0.07 & 94 $\pm$ 0.36 \\
50   & 0.37 $\pm$ 0.01 & 2.45 $\pm$ 0.05 & 10.0 $\pm$ 0.09 & 41 $\pm$ 0.25 & 169 $\pm$ 4.56 \\
100  & 0.71 $\pm$ 0.07 & 4.66 $\pm$ 0.09 & 19.1 $\pm$ 0.23 & 77 $\pm$ 0.46 & 314 $\pm$ 2.54 \\
200  & 1.50 $\pm$ 0.10 & 9.62 $\pm$ 0.16 & 39.4 $\pm$ 0.41 & 158 $\pm$ 3.19 & 625 $\pm$ 2.88 \\
\hline
\end{tabular}

\caption{Computation time (seconds) for vine copula fitting across dimensions $D$ and sample sizes, on a MacBook Pro M1 Pro (2023). The results were averaged over 10 different datasets, per dimension/datasize pair.}
\label{tab:copula_fitting_times}
\end{table}

\section{MODELS}
\label{sec:models}

We provide details of the models evaluated in our paper.

\paragraph{Engression} Engression, proposed in \cite{shen2023engression}, approximates the conditional distribution $Y\mid X$ using a pre-additive noise model $Y = g(WX + \eta) + \beta^\top X$, where $g: \mathbb{R}^d \rightarrow \mathbb{R}$ is a non-linear function that captures non-linear relationships and $\eta = h(\epsilon)$ introduces flexible noise. Built on a neural network architecture that efficiently learns this structure, it optimizes the energy score loss for accurate distributional regression.

\paragraph{Meta-learners}
Meta-learners are flexible frameworks in causal inference designed to estimate individualized treatment effects by leveraging machine learning models. Two common types are T-learners and S-learners. Details can be found in \cite{kunzel2019metalearners}.

T-learners work by training separate models for the treated and untreated groups, predicting outcomes under each treatment condition, and then calculating the difference between these predictions to estimate the treatment effect.
S-learners combine both treated and untreated data into a single model by including treatment as an input feature, allowing the model to learn the outcome function across both treatment conditions simultaneously.
These learners provide a modular approach to estimating conditional average treatment effects (CATE) and can adapt to different settings and model complexities.
\paragraph{CausalForest}

CausalForest is an extension of random forests designed to estimate heterogeneous treatment effects by partitioning the data into subgroups with similar treatment responses. Introduced by \cite{wager2018estimation}, it uses a tree-based ensemble method to non-parametrically estimate the CATE by building separate models for different covariate regions, while ensuring a balance between treated and control units in each partition. This method is flexible and adapts to complex data structures, making it a powerful tool for understanding treatment effect heterogeneity.

\paragraph{BART} Bayesian Additive Regression Trees, first introduced in  \cite{chipman2010bart}, is a non-parametric machine learning method that uses an ensemble of regression trees to model complex relationships between covariates and outcomes.  The BART model estimates the posterior distribution of the outcome by summing the contributions from many trees, each of which is trained to explain part of the residual error left by the others. This ensemble approach makes BART particularly effective at capturing complex, non-linear relationships between the covariates and the outcome. Unlike standard decision trees, BART applies a Bayesian framework, allowing it to quantify uncertainty in its predictions and avoid overfitting through regularization priors.

\paragraph{TARNet} Treatment-Agnostic Representation Network, first introduced in \cite{johansson2016learning}, is a neural network-based model for estimating heterogeneous treatment effects in causal inference. It works by learning a shared representation of covariates, independent of treatment assignment, and then using this representation to estimate potential outcomes for both the treated and untreated groups. By focusing on treatment-agnostic representation learning, TARNet aims to improve the generalizability and accuracy of treatment effect estimates, particularly in high-dimensional settings.

\section{COMPUTATION DETAILS}
\label{sec:computation_details}
We provide computation details in \Cref{sec:experiments}. We use default recommended hyperparameters for each model.

\begin{table}[h]
\caption{Hyperparameters of Each Model.} 
\label{tab:hyperparameter}
\begin{center}
\begin{tabular}{l|p{6cm}|p{5cm}}
\toprule
\textbf{Model} & \textbf{Key Hyperparameters} & \textbf{Package} \\
\midrule
TARNet & \begin{minipage}{5.5cm} number of layers = 2\\
batch size = 64\\
learning rate = 0.0001\\
number of epochs = 2000 \end{minipage} & \begin{minipage}{5cm} Python\\ \texttt{catenets} \citep{curth2021really} \end{minipage}\\
\cmidrule{1-3}
CausalForest & \begin{minipage}{5.5cm} number of trees = 100\\
maximum depth = 3
\end{minipage} & \begin{minipage}{5.5cm}
Python, \texttt{econml}\\ \citep{econml} \end{minipage}\\
\cmidrule{1-3}
S-/T-BART & \begin{minipage}{5.5cm} number of trees = 75\\ number of iterations = 4\\ 
number of burn-in iterations = 200\\ posterior draws = 800
\end{minipage} & R, \texttt{dbarts} \citep{dbarts} \\
\cmidrule{1-3}
S-/T-engression & \begin{minipage}{5.5cm} number of layers = 3\\ 
batch size = 64\\ 
learning rate = 0.01 \\
number of epochs = 500 \end{minipage} & \begin{minipage}{5cm}
Python, \texttt{engression}\\ \citep{engression}\end{minipage}\\
\bottomrule
\end{tabular}
\end{center}
\end{table}

All experiments were conducted on a MacBook with an Apple M3 chip, 8-core CPU, and 32GB RAM.

\section{ADDITIONAL EXPERIMENTS}
\label{sec:additional_exp}
\subsection{IN-DOMAIN MODEL PERFORMANCE TESTING ON THE IHDP DATASET}
\label{sec:indomain}
Although our proposed method mainly tackles the out-of-domain generalizability assessment, which is a challenging task as demonstrated in \Cref{sec:generalizability_in_causal_inference},  it can be easily adapted to performance evaluation for in-domain tasks. As an illustration, we present the in-domain test results and MSE for the IHDP dataset, using the same experimental setup as in \Cref{sec:IHDP} but without introducing any domain shift (i.e.~$Z_1$ remains unchanged in the test domain) in \Cref{fig:ihdp_indomain}.

\begin{figure}[H]
\vspace{.3in}
\centerline{\includegraphics[width=0.8\linewidth]{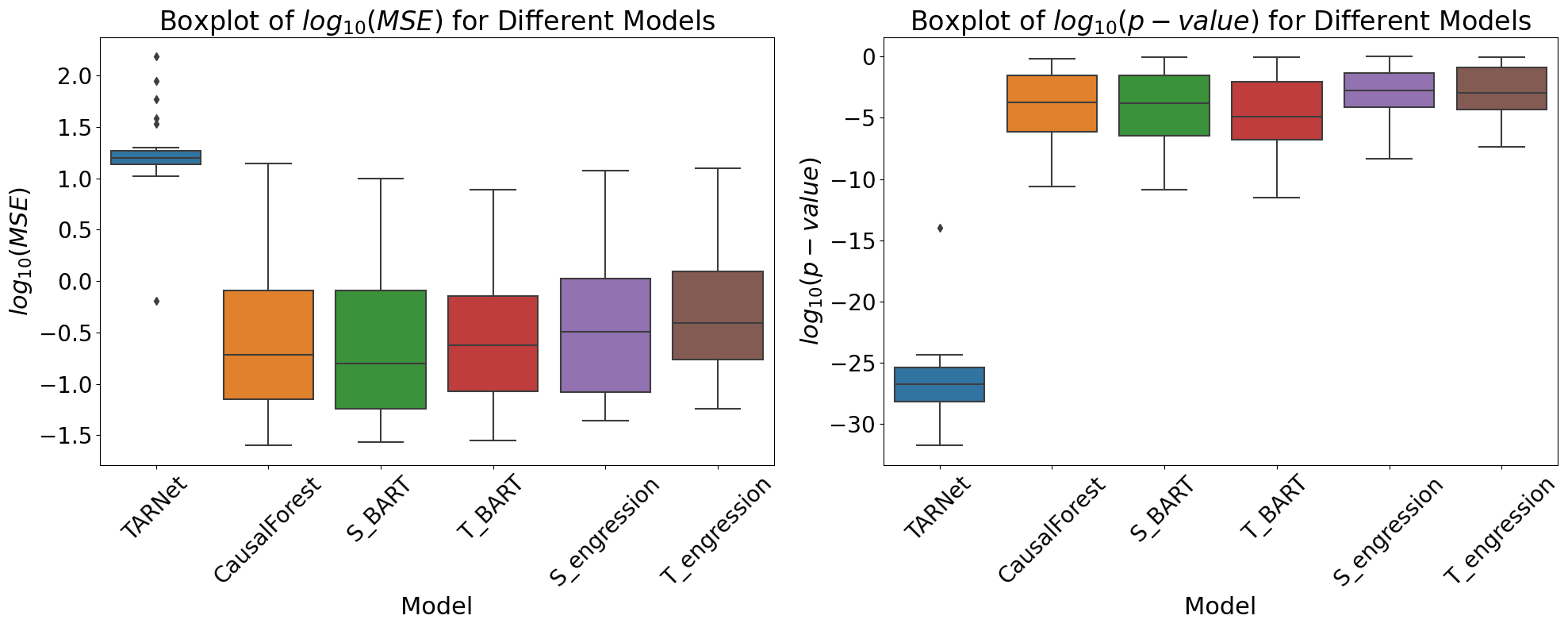}}
\vspace{.3in}
\caption{$\log_{10}(\text{MSE})$ and $\log_{10}(\text{p-value})$ of Mean Regression Testing on the IHDP Dataset, No Domain Shift.}
\label{fig:ihdp_indomain}
\end{figure}

\Cref{fig:ihdp_indomain} demonstrates the contrasts of $\log_{10}(\text{MSE})$ and $\log_{10}(\text{p-value})$ performance assessment results.Each model was trained with its default hyperparameters, and we evaluated them under those same conditions. The test results therefore reflect each model’s generalizability given its default settings. As expected, we see alignments of the MSE and tests results: TARNet (with default hyper-parameter settings) exhibits large MSE, and the p-values are generally very small. Meanwhile, S-engression and T-engression yield comparatively lower MSEs; however, MSE alone can be insufficiently persuasive. By incorporating p-values and the corresponding statistical guarantees offered by our method, we can make stronger assertions about the generalizability of these two engression approaches. These findings emphasize the usefulness and significance of our proposed method in model assessment, as discussed at the end of \Cref{sec:generalizability_in_causal_inference}.

\subsection{TESTING GENERALIZABLE MODELS}
\label{sec:linear}
We include an additional experiment in this section, which is based on the synthetic data setting in \Cref{sec:synthetic}, but without domain shift. We set the marginal distribution of $Z_1$, $Z_2$ to be $\mathcal{N}(1,1)$, and $Y(X) \sim \mathcal{N}(2X+1,1)$, $X\sim \operatorname{Bernoulli} (0.5)$. In this case, the conditional average treatment effect should be linear. 

The result when there is no domain shift can be found in \Cref{fig:synthetic_mean_p_noshift}. We see that the p-values of both S-Linear (Regression) and T-Linear (Regression) are uniformly distributed. Given the true CATE function is indeed linear, this result validates our proposed method.

\begin{figure}[H]
\vspace{.3in}
\centerline{\includegraphics[width=0.5\linewidth]{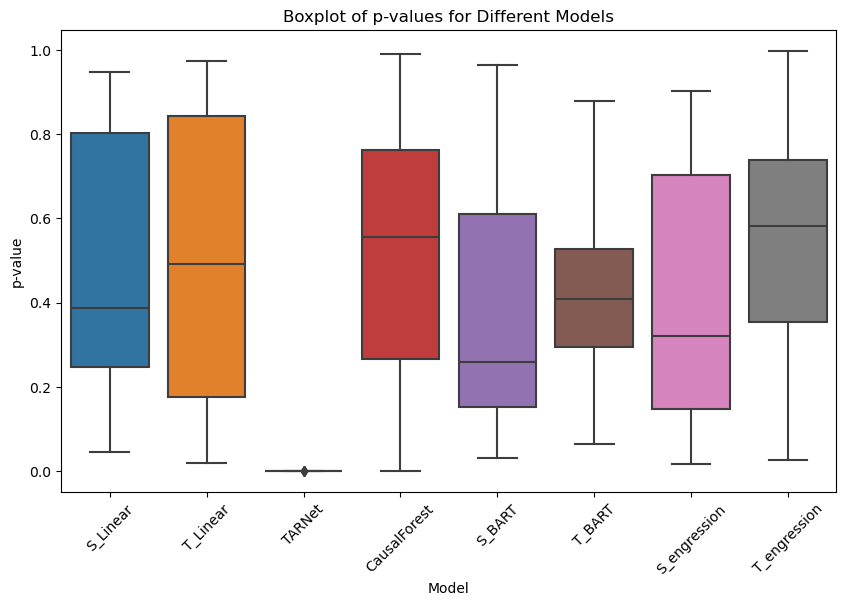}}
\vspace{.3in}
\caption{$p$-values of Mean Regression Testing, Synthetic Data of 50 Iterations, No Domain Shift.}
\label{fig:synthetic_mean_p_noshift}
\end{figure}

We next test when there is domain shift, i.e.~we keep all the settings the same as above for training set, but we change the marginal distribution of $Z_1$, $Z_2$ in the test set to be $\mathcal{N}(3,2)$. \Cref{fig:synthetic_mean_p_shift} shows the results. Linear regressions still demonstrate good generalizability performance! However for algorithms like S-engression and S-BART the results worsen, likely due to problems such as overfitting.

\begin{figure}[H]
\vspace{.3in}
\centerline{\includegraphics[width=0.5\linewidth]{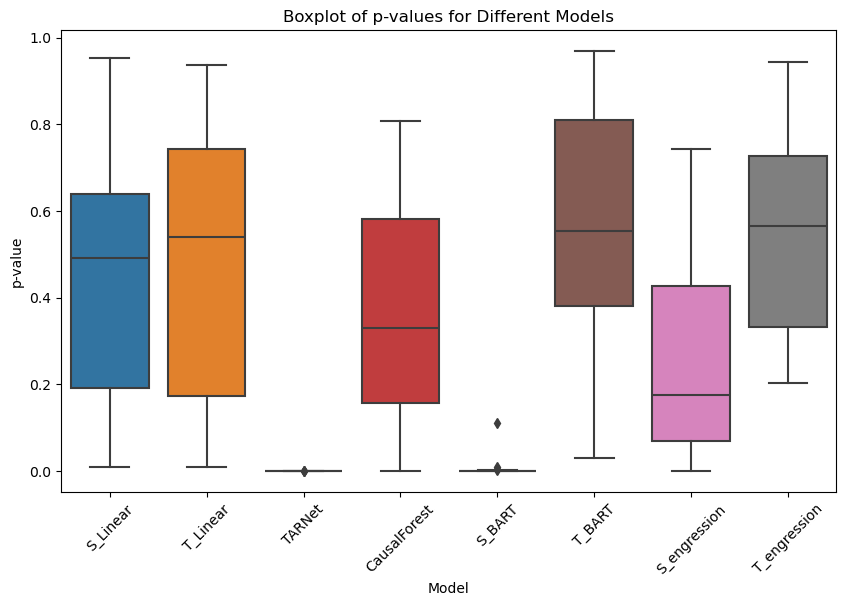}}
\vspace{.3in}
\caption{$p$-values of Mean Regression Testing, Synthetic Data of 50 Iterations, with Domain Shift.}
\label{fig:synthetic_mean_p_shift}
\end{figure}

\subsection{MORE COMPLICATED DATA GENERATION}
\label{sec:complicated_exp}
 To demonstrate the flexibility of our approach, we run additional experiments across different data generation settings, including increasing number of covariates, changing marginal distributions and changing dependency structures.
 
Tables \ref{tab:50_covariates} and \ref{tab:100_covariates} show the $\log_{10}(\text{p-value})$ statistics from 50 trials conducted under a setting similar to Synthetic Setting 1 in the main body of our paper, with only two changes: (1) we increase the number of covariates from 2 to 50 and 100, and keep the covariate distribution shifts the same for each covariate; (2) we replace the dependency structures with randomly sampled correlation matrices. CausalForest, T-engression and T-BART demonstrate good generalizability in these settings.

\begin{table}[H]
\centering
\caption{$\log_{10}(p\text{-values})$ Statistics under Synthetic Setting 1 with 50 Covariates.}
\begin{tabular}{rrrrrrr} 
\toprule 
\textbf{Model} & \textbf{Min} & \textbf{25\%} & \textbf{Median} & \textbf{Mean} & \textbf{75\%} & \textbf{Max} \\
\midrule 
TARNet & $-36.8$ & $-33.6$ & $-32.4$ & $-31.6$ & $-31.5$ & $-30.8$ \\ 
CausalForest & $-2.35$ & $-1.22$ & $-0.851$ & $-0.539$ & $-0.214$ & $-0.077$ \\ 
S-BART & $-8.32$ & $-4.22$ & $-3.36$ & $-2.58$ & $-2.66$ & $-1.92$ \\ 
T-BART & $-1.52$ & $-0.757$ & $-0.326$ & $-0.349$ & $-0.187$ & $-0.044$ \\ 
S-engression & $-20.0$ & $-18.2$ & $-17.3$ & $-14.4$ & $-16.7$ & $-13.1$ \\ 
T-engression & $-2.53$ & $-0.669$ & $-0.283$ & $-0.324$ & $-0.211$ & $-0.006$ \\ 
\bottomrule
\end{tabular}
\label{tab:50_covariates}
\end{table}

\begin{table}[H]
\centering
\caption{$\log_{10}(p\text{-values})$ Statistics under Synthetic Setting 1 with 100 Covariates.}
\begin{tabular}{rrrrrrr} \toprule 
\textbf{Model} & \textbf{Min} & \textbf{25\%} & \textbf{Median} & \textbf{Mean} & \textbf{75\%} & \textbf{Max} \\
\midrule 
 TARNet & $-34.3$ & $-31.3$ & $-30.9$ & $-29.9$ & $-30.3$ & $-29.0$ \\ 
 CausalForest & $-2.39$ & $-1.35$ & $-0.821$ & $-0.670$ & $-0.479$ & $-0.122$ \\ 
 S-BART & $-9.62$ & $-7.32$ & $-6.68$ & $-5.22$ & $-5.99$ & $-3.96$ \\ 
 T-BART & $-1.04$ & $-0.76$ & $-0.36$ & $-0.31$ & $-0.12$ & $0.00$ \\ 
 S-engression & $-28.6$ & $-26.1$ & $-25.3$ & $-23.6$ & $-24.0$ & $-22.6$ \\ 
 T-engression & $-2.46$ & $-0.663$ & $-0.393$ & $-0.366$ & $-0.137$ & $-0.107$ \\ 
 \bottomrule
 \end{tabular}
\label{tab:100_covariates}
\end{table}

\Cref{tab:non-linear} present the $\log_{10}(p\text{-values})$ statistics from 50 trials under the same setup as the first experiment in \ref{sec:linear} except for altering the marginal causal distribution. Changing this from Gaussian to gamma introduces non-linear dependencies in the conditional causal margin. While linear regression was generalizable in the original setup, it fails in the non-linear setting, demonstrating the ability of our approach to show that some methods fail to generalize well.

\begin{table}[H]
\centering
\caption{$\log_{10}(p\text{-values})$ Statistics under the Same Set-up as \Cref{fig:synthetic_mean_p_noshift} with Non-linear Dependency. }
\begin{tabular}{rrrrrrr} 
\toprule 
\textbf{Model} & \textbf{Min} & \textbf{25\%} & \textbf{Median} & \textbf{Mean} & \textbf{75\%} & \textbf{Max} \\
\midrule 
S-Linear & $-16.2$ & $-13.0$ & $-11.7$ & $-10.9$ & $-10.8$ & $-9.93$ \\ 
T-Linear & $-13.3$ & $-11.2$ & $-10.6$ & $-9.67$ & $-10.0$ & $-8.61$ \\ 
TARNet & $-24.0$ & $-21.7$ & $-21.2$ & $-19.5$ &$ -19.9$ & $-18.6$ \\ 
CausalForest & $-12.2$ & $-10.8$ & $-10.2$ & $-9.16$ & $-9.34$ & $-8.23$ \\ 
S-BART & $-13.4$ & $-10.3$ & $-9.33$ & $-7.60$ & $-8.50$ & $-6.36$ \\ 
T-BART & $-11.6$ & $-8.40$ & $-7.84$ & $-6.38$ & $-7.49$ & $-5.12$ \\ 
S-engression & $-12.5$ & $-9.89$ & $-9.45$ & $-8.01$ & $-8.17$ & $-7.13$ \\ 
T-engression & $-9.69$ & $-7.32$ & $-6.83$ & $-5.23$ & $-6.32$ & $-3.99$ \\ 
\bottomrule
\end{tabular}
\label{tab:pvalue_statistics}

\label{tab:non-linear}
\end{table}

A strength of our framework is that vine copula allows users to test their methods against various classes of copulas. We demonstrate this in \Cref{tab:non-gaussian_copula} with the following data generating process:

\begin{itemize}
    \item Training Domain: Covariates' marginal distributions are identical gamma distributions with shape $k=8$ and rate $\theta=4$;
    \item Testing Domain: Covariates' marginal distributions are identical gamma distributions with shape $k=2$ and rate $\theta=1$;
    \item Marginal Causal Distribution: Modelled as an exponential distribution with $k=0.5x+0.1$;
    \item Treatment Assignment: Specified as $ X\sim \operatorname{Bernoulli} (0.5)$;
    \item Copula: Randomly sampled R-vine structure, with each bivariate copula set to be a Clayton copula \citep{kreinovich2013clayton} with a parameter of 2. 
\end{itemize}

\begin{table}[H]
\caption{$\log_{10}(p\text{-values})$ Statistics for Experiment with a Non-Gaussian Copula. }
\centering
\begin{tabular}{rrrrrrr} 
\toprule 
\textbf{Model} & \textbf{Min} & \textbf{25\%} & \textbf{Median} & \textbf{Mean} & \textbf{75\%} & \textbf{Max} \\ 
\midrule 
S-Linear & $-\infty$ & $-5.71$ & $-4.92$ & $-3.89$ & $-4.15$ & $-2.81$ \\ T-Linear & $-\infty$ & $-3.47$ & $-2.64$ & $-1.87$ & $-1.94$ & $-0.929$ \\ TARNet & $-18.3$ & $-15.9$ & $-14.5$ & $-12.5$ & $-13.7$ & $-11.2$ \\ CausalForest & $-10.9$ & $-3.53$ & $-2.81$ & $-2.35$ & $-2.23$ & $-1.49$ \\ S-BART & $-\infty$ & $-4.12$ & $-3.47$ & $-2.91$ & $-2.99$ & $-2.04$ \\ T-BART & $-\infty$ & $-4.14$ & $-3.24$ & $-2.62$ & $-2.62$ & $-1.73$ \\ S-engression & $-15.8$ & $-4.06$ & $-3.04$ & $-2.35$ & $-2.63$ & $-1.54$ \\ T-engression & $-10.2$ & $-3.70$ & $-2.28$ & $-1.95$ & $-1.70$ & $-1.40$ \\ 
\bottomrule
\end{tabular}

\label{tab:non-gaussian_copula}
\end{table}

 \Cref{tab:gaussian_copula} shows the $\log_{10}(p\text{-values})$ of testing generalizability results with data generated from a Gaussian copula. The covariate margins, the causal margins, the dependency structure, and the second moments of each bivariate copula are identical to the previous example. We choose the rank correlation coefficient of the Gaussian copula, $
\rho = \frac{\theta}{2+\theta}$, where $\theta$ parameterizes the Clayton copula; this was set as 2 in the previous example. The only difference between the two processes is the class of the copula family. The $-\infty$ in Tables \ref{tab:non-gaussian_copula} and \ref{tab:gaussian_copula} are due to the original $p$-values being 0. 

\begin{table}[H]
\centering
\caption{$\log_{10}(p\text{-values})$ for Experiment with the Same Setting as in \Cref{tab:non-gaussian_copula}, but with a Gaussian Coupla.}
\begin{tabular}{rrrrrrr} 
\toprule 
\textbf{Model} & \textbf{Min} & \textbf{25\%} & \textbf{Median} & \textbf{Mean} & \textbf{75\%} & \textbf{Max} \\
\midrule 
S-Linear & $-\infty$ & $-5.19$ & $-4.68$ & $-3.52$ & $-3.49$ & $-2.73$ \\ T-Linear & $-\infty$ & $-2.59$ & $-1.83$ & $-1.42$ & $-1.45$ & $-0.512$ \\ TARNet & $-21.3$ & $-15.6$ & $-14.3$ & $-13.5$ & $-13.6$ & $-12.7$ \\ CausalForest & $-5.45$ & $-3.81$ & $-2.94$ & $-1.70$ & $-2.01$ & $-0.517$ \\ S-BART & $-\infty$ & $-3.48$ & $-2.85$ & $-2.13$ & $-2.09$ & $-1.20$ \\ T-BART & $-\infty$ & $-2.85$ & $-2.17$ & $-1.74$ & $-1.44$ & $-1.07$ \\ S-engression & $-11.4$ & $-3.16$ & $-2.62$ & $-1.82$ & $-1.68$ & $-1.10$ \\ T-engression & $-9.24$ & $-2.61$ & $-1.51$ & $-1.08$ & $-0.921$ & $-0.322$ \\ 
\bottomrule
\end{tabular}

\label{tab:gaussian_copula}
\end{table}

Contrasting Tables \ref{tab:non-gaussian_copula} and \ref{tab:gaussian_copula} shows that model generalizability is sensitive to copula families. Therefore, the flexibility of simulating data from different copula families, which is a key advantage of our current parametric framework, is important for model generalizability evaluation. We would also like to emphasize that in this paper we simulate from frugal models parametrically, but there are methods that which can flexibly model copulas without parametric assumptions \citep{de2024marginal}, and others may not require copulas at all.

\subsection{Equivalence Testing}
\label{sec:equiv_testing}
Our framework is flexible and naturally accommodates equivalence testing. Note that equivalence testing can be restrictive due to its need to define an additional hyperparameter---the equivalence margin---which can influence test outcomes. However, in certain applications, such as those requiring guarantees about not overlooking non-generalizable models, equivalence testing (e.g.~TOST: two one-sided tests) can be more appropriate. Here, the null hypothesis becomes $H_0$: $|\hat{\tau}^B-\tau^B|\geq \delta$, and the Type I error corresponds to the risk of falsely concluding that the model generalizes. We provide additional experiment results on the synthetic datasets. 

On synthetic data, we report TOST results for two margins, $\delta=0.1$ and $\delta=0.2$, using the same bootstrap configuration
as \Cref{alg:mean_test_algo} ($N_A=200$, $N_B=50$, $N_{btp}=200$, 50 repetitions).

\begin{table}[ht]
\centering
\caption{Synthetic setting 1, $\delta = 0.1$ (TOST $p$-values)}\label{tab:tost_syn1_d01}
\begin{tabular}{lcccc}
\toprule
\textbf{Model} & \textbf{Min} & \textbf{Median} & \textbf{Mean} & \textbf{Max}\\
\midrule
TARNet        & 1.00 & 1.00 & 1.00 & 1.00 \\
CausalForest  & $1.00\times10^{-6}$ & $1.75\times10^{-4}$ & $8.49\times10^{-3}$ & 0.119 \\
S-BART       & $4.00\times10^{-6}$ & $3.88\times10^{-4}$ & $9.04\times10^{-3}$ & 0.126 \\
T-BART       & $1.83\times10^{-4}$ & $4.86\times10^{-3}$ & $3.53\times10^{-2}$ & 0.381 \\
S-engression & $3.72\times10^{-4}$ & $8.13\times10^{-2}$ & 0.133 & 0.717 \\
T-engression & $7.70\times10^{-4}$ & $2.31\times10^{-2}$ & $3.93\times10^{-2}$ & 0.226 \\
\bottomrule
\end{tabular}
\label{tab:synthetic1_d01_equiv}
\end{table}

\begin{table}[ht]
\centering
\caption{Synthetic setting 1, $\delta = 0.2$ (TOST p-values)}\label{tab:tost_syn1_d02}
\begin{tabular}{lcccc}
\toprule
\textbf{Model} & \textbf{Min} & \textbf{Median} & \textbf{Mean} & \textbf{Max}\\
\midrule
TARNet        & 1.00 & 1.00 & 1.00 & 1.00 \\
CausalForest  & $5.55\times10^{-16}$ & $1.20\times10^{-10}$ & $5.10\times10^{-9}$ & $1.51\times10^{-7}$ \\
S-BART       & $1.65\times10^{-14}$ & $6.18\times10^{-11}$ & $8.02\times10^{-8}$ & $1.60\times10^{-6}$ \\
T-BART       & $2.57\times10^{-10}$ & $6.45\times10^{-8}$ & $6.25\times10^{-5}$ & $1.96\times10^{-3}$ \\
S-engression & $1.66\times10^{-11}$ & $3.31\times10^{-6}$ & $1.26\times10^{-4}$ & $2.87\times10^{-3}$ \\
T-engression & $2.76\times10^{-9}$  & $1.73\times10^{-5}$ & $2.64\times10^{-4}$ & $2.80\times10^{-3}$ \\
\bottomrule
\end{tabular}
\label{tab:synthetic1_d02_equiv}
\end{table}

 Setting $\delta =0.2$ means the null hypothesis allows for a wider range of acceptable discrepancy than $\delta = 0.1$.  As a result, we expect higher rejection rates (or smaller p-values) when $\delta = 0.2$, since the criterion for equivalence is more lenient.  Conversely, with $\delta = 0.1$, the test is stricter, and p-values are generally larger. This is validated by comparing the results in Tables \ref{tab:synthetic1_d01_equiv} and \ref{tab:synthetic1_d02_equiv}.

\begin{table}[ht]
\centering
\caption{Synthetic setting 2, $\delta = 0.1$ (TOST p-values)}\label{tab:tost_syn2_d01}
\begin{tabular}{lcccc}
\toprule
\textbf{Model} & \textbf{Min} & \textbf{Median} & \textbf{Mean} & \textbf{Max}\\
\midrule
TARNet        & 1.00 & 1.00 & 1.00 & 1.00 \\
CausalForest  & $4.00\times10^{-6}$ & $1.60\times10^{-3}$ & $1.99\times10^{-2}$ & 0.336 \\
S-BART       & $1.00\times10^{-5}$ & $2.21\times10^{-3}$ & $3.08\times10^{-2}$ & 0.391 \\
T-BART       & $2.07\times10^{-4}$ & $4.54\times10^{-3}$ & $2.92\times10^{-2}$ & 0.282 \\
S-engression & $9.38\times10^{-4}$ & 0.108 & 0.191 & 0.745 \\
T-engression & $2.30\times10^{-2}$ & 0.109 & 0.162 & 0.381 \\
\bottomrule
\end{tabular}
\label{tab:synthetic2_d01_equiv}
\end{table}

\begin{table}[ht]
\centering
\caption{Synthetic setting 2, $\delta = 0.2$ (TOST p-values)}\label{tab:tost_syn2_d02}
\begin{tabular}{lcccc}
\toprule
\textbf{Model} & \textbf{Min} & \textbf{Median} & \textbf{Mean} & \textbf{Max}\\
\midrule
TARNet        & 1.00 & 1.00 & 1.00 & 1.00 \\
CausalForest  & $3.64\times10^{-12}$ & $2.64\times10^{-10}$ & $1.08\times10^{-7}$ & $2.00\times10^{-6}$ \\
S-BART       & $4.10\times10^{-6}$  & $8.02\times10^{-3}$  & $4.32\times10^{-2}$ & 0.275 \\
T-BART       & $1.02\times10^{-3}$  & $2.67\times10^{-2}$  & $5.43\times10^{-2}$ & 0.329 \\
S-engression & $3.03\times10^{-6}$  & $1.36\times10^{-3}$  & $2.26\times10^{-2}$ & 0.222 \\
T-engression & $1.85\times10^{-3}$  & $4.77\times10^{-2}$  & $6.37\times10^{-2}$ & 0.214 \\
\bottomrule
\end{tabular}
\label{tab:synthetic2_d02_equiv}
\end{table}

Setting 2 involves larger domain shifts than Setting 1, making model generalizability more challenging. As expected, results in Tables \ref{tab:synthetic2_d01_equiv} and \ref{tab:synthetic2_d02_equiv} show generally higher p-values under Setting 2, reflecting the difficulty in rejecting the null hypothesis of non-equivalence and non-generalizability.

We also ran equivalence testing under the same setting as in \Cref{fig:synthetic_mean_p_noshift}, where Linear Regression models are expected to exhibit perfect transportability. Accordingly, when setting $\delta = 0.1$, we obtain the results in \Cref{tab:linear_equiv}. These are exactly as expected---models should reject the null hypothesis in this setting, confirming their strong transportability under equivalence testing.

\begin{table}[ht]
\centering
\caption{Linear-regression transportability, synthetic ($\delta = 0.1$)}\label{tab:tost_linreg}
\begin{tabular}{lcccc}
\toprule
\textbf{Model} & \textbf{Min} & \textbf{Median} & \textbf{Mean} & \textbf{Max}\\
\midrule
S-Linear & $1.47\times10^{-13}$ & $1.47\times10^{-10}$ & $1.28\times10^{-8}$ & $2.17\times10^{-7}$ \\
T-Linear & $1.54\times10^{-11}$ & $1.10\times10^{-8}$  & $4.57\times10^{-7}$ & $9.10\times10^{-6}$ \\
\bottomrule
\end{tabular}
\label{tab:linear_equiv}
\end{table}


\section{INTERPRETING TESTING RESULTS}
\label{sec:read_p}
We further explain the motivation of our paper, as well as guidance of reading the testing results.

All p-values, including their distributions, are highly informative in evaluating generalizability. For example, consistently small p-values (as shown in \Cref{fig:ihdp_mean}), indicate a clear failure of model generalizability in that setting. Conversely, uniform distributions of p-values (e.g.~linear regression results in \Cref{fig:synthetic_mean_p_noshift}) demonstrate more trust in the model’s generalizability. Type-I error control serves a critical role in distinguishing between competing hypotheses with a minimal probability of error. In our framework, controlling Type-I error ensures that conclusions about non-generalization when a model fails the test are not driven by random noise. This rigour is crucial for causal inference, where decisions based on incorrect conclusions can have significant consequences. In contrast, predictive performance measures like MSE lack statistical safeguards, and interpretations of model performance under domain shifts would lack reliability and robustness.

We also provide explanations if all tests fail. As with any hypothesis test, failing to pass provides evidence against the tested hypothesis. In our framework, this means the algorithm lacks sufficient generalizability to infer the conditional treatment margin in new domains. If all algorithms fail, it signals none are suitable for reliable causal inference under the domain shift.

This highlights the need for alternative modelling approaches and underscores the value of our framework. Unlike MSE, which compares predictive performance, our method directly identifies failures in causal generalizability—an essential insight for researchers. We hope this clarifies how to interpret such results and guides researchers in determining next steps when all models fail.

\section{Comparison with scores}
\label{sec:compare_with_MSE}

Our testing framework is actionable in that it delivers a
principled, binary decision on whether a model is
\emph{generalizable} to a given domain.  This is essential for model
selection: rather than relying on metrics such as mean square error alone (which may
favour non-generalizable models), we first use our test to \emph{filter
out} models that fail to generalize.

Our method structures selection into two stages:
\begin{description}
  \item[\textbf{Stage~1:}] Apply the proposed testing procedure to
        identify models that generalize across domains.
  \item[\textbf{Stage~2:}] Among the models that pass the test,
        rank them with a predictive metric (such as MSE) and pick the
        best-performing one.
\end{description}

This two-stage approach ensures selection is both statistically sound
and practically robust: it prioritizes generalizability before
performance.  In this framework we choose models that are
``good and generalizable,'' not merely ``relatively good'' by the score 
alone.  MSE is actionable only in the sense that it lets one compare
already viable models and hyper-parameter settings.

We do not argue against continuous score; instead, we view them and our test as
complementary.  Our test provides statistical guarantees on
generalizability.  Once generalizable models are identified, MSE can rank
their relative predictive performance.  Relying on a continuous score alone is
insufficient---a model may achieve a comparatively low score in one domain, yet fail to
generalize elsewhere.

To highlight the discrepancy, we re-ran the experiments in the paper and
recorded mean square error across 50 trials.  Tables~\ref{tab:syn1},
\ref{tab:syn2}, and~\ref{tab:real} show the minimum and maximum MSE
values and the corresponding performance ranks (lower is better) for
each model.

\begin{table}[ht]
  \centering
  \caption{Synthetic Setting~1: MSE statistics over 50 trials}
  \label{tab:syn1}
  \begin{tabular}{lcccc}
    \toprule
    \textbf{Model} & \textbf{Min} & \textbf{Max} & \textbf{Min Rank} & \textbf{Max Rank}\\
    \midrule
    TARNet         & 2.40 & 2.68 & 6 & 6 \\
    CausalForest   & 0.001 & 0.041 & 1 & 4 \\
    S-BART          & 0.040 & 0.067 & 1 & 4 \\
    T-BART          & 0.004 & 0.130 & 1 & 5 \\
    S-engression   & 0.016 & 0.144 & 2 & 5 \\
    T-engression   & 0.006 & 0.080 & 1 & 5 \\
    \bottomrule
  \end{tabular}
\end{table}

\begin{table}[ht]
  \centering
  \caption{Synthetic Setting~2: MSE statistics over 50 trials}
  \label{tab:syn2}
  \begin{tabular}{lcccc}
    \toprule
    \textbf{Model} & \textbf{Min} & \textbf{Max} & \textbf{Min Rank} & \textbf{Max Rank}\\
    \midrule
    TARNet         & 2.09  & 2.92  & 5 & 6 \\
    CausalForest   & 0.012 & 0.230 & 1 & 3 \\
    S-BART          & 0.040 & 0.150 & 2 & 5 \\
    T-BART          & 0.030 & 0.200 & 1 & 3 \\
    S-engression   & 0.120 & 0.600 & 4 & 6 \\
    T-engression   & 0.020 & 0.180 & 1 & 4 \\
    \bottomrule
  \end{tabular}
\end{table}

\begin{table}[ht]
  \centering
  \caption{IHDP: MSE statistics over 50 trials}
  \label{tab:real}
  \begin{tabular}{lcccc}
    \toprule
    \textbf{Model} & \textbf{Min} & \textbf{Max} & \textbf{Min Rank} & \textbf{Max Rank}\\
    \midrule
    TARNet         & 10.2 & 86.0 & 6 & 6 \\
    CausalForest   & 0.03  & 6.19  & 1 & 5 \\
    S-BART          & 0.03  & 6.72  & 1 & 4 \\
    T-BART          & 0.02  & 6.25  & 1 & 3 \\
    S-engression   & 0.10  & 10.16 & 2 & 5 \\
    T-engression   & 0.05  & 6.45  & 1 & 5 \\
    \bottomrule
  \end{tabular}
\end{table}

MSE and rank summaries provide no statistical confidence that a
model generalizes.  For example, S-BART's minimum MSE in Synthetic setting 2
is $0.04$, far smaller than TARNet's, yet this does not prove S-BART
generalizes.  In contrast, \Cref{fig:synthetic_mean_p} of our paper shows small $p$-values for
S-BART, letting us reject the null hypothesis of generalizability at the 5\% level. Another limitation of MSE is that in heterogeneous or endogenous noise
settings, cross-domain MSEs may diverge even with a perfectly specified
CATE model.  Differing noise levels alone can create apparent
performance gaps.

\newpage

\end{document}